\documentclass[final]{elsarticle}

\usepackage{amssymb,amsmath,amsthm}
\usepackage{hyperref}
\usepackage{array,multirow}
\usepackage[utf8]{inputenc}
\usepackage[english]{babel}
\usepackage{caption}
\usepackage{subcaption}
\usepackage{floatrow}
\usepackage[nolist]{acronym}
\usepackage{placeins}
\usepackage{algorithmicx}
\usepackage{algorithm}
\usepackage{algpseudocode}
\newtheorem{theorem}{Theorem}[section]

\newtheorem{lemma}[theorem]{Lemma}
\DeclareMathOperator*{\argmin}{arg\,min}
\usepackage{verbatim}
\usepackage[load-configurations=version-1]{siunitx} 

\newcounter{phase}[algorithm]
\newlength{\phaserulewidth}
\newcommand{\setphaserulewidth}{\setlength{\phaserulewidth}}
\newcommand{\phase}[1]{%
  \Statex\leavevmode\llap{\rule{\dimexpr\labelwidth+\labelsep}{\phaserulewidth}}\rule{\linewidth}{\phaserulewidth}
  \Statex\strut\refstepcounter{phase}\textit{State~\thephase~--~#1}
  \vspace{-5pt}
  \Statex\leavevmode\llap{\rule{\dimexpr\labelwidth+\labelsep}{\phaserulewidth}}\rule{\linewidth}{\phaserulewidth}}

\newcommand{\phasebegin}[1]{%
  \Statex\strut\refstepcounter{phase}\textit{State~\thephase~--~#1}
  \Statex\leavevmode\llap{\rule{\dimexpr\labelwidth+\labelsep}{\phaserulewidth}}\rule{\linewidth}{\phaserulewidth}}
\makeatother

\setphaserulewidth{.2pt}

\begin{document}

\def\*#1{\mathbf{#1}}
\def\b#1{\boldsymbol{#1}}

\newcommand{\specialcell}[2][c]{%
  \begin{tabular}[#1]{@{}c@{}}#2\end{tabular}}
\newcolumntype{H}{>{\setbox0=\hbox\bgroup}c<{\egroup}@{}}
\newcommand{\norm}[1]{\left\lVert#1\right\rVert}
\begin{frontmatter}

\title{Reactive Soft Prototype Computing\\ for Concept Drift Streams\footnote{The final authenticated publication is available at \newline \url{https://doi.org/10.1016/j.neucom.2019.11.111}}}

\author[label1]{Christoph Raab\corref{cor1}}
\ead{christoph.raab@fhws.de}

\author[label1]{Moritz Heusinger}
\ead{moritz.heusinger@fhws.de}

\author[label1]{Frank-Michael Schleif}
\ead{frank.schleif@fhws.de}

\address[label1]{University for Applied Sciences Würzburg-Schweinfurt, Sanderheinrichsleitenweg 20, Würzburg, Germany}
\cortext[cor1]{I am corresponding author}

\begin{abstract}
    The amount of real-time communication between agents in an information system has increased rapidly since the beginning of the decade. This is because the use of these systems, e.g.\, social media, has become commonplace in today's society. This requires analytical algorithms to learn and predict this stream of information in real-time.
    The nature of these systems is non-static and can be explained, among other things, by the fast pace of trends.
    This creates an environment in which algorithms must recognize changes and adapt.
    Recent work shows vital research in the field, but mainly lack stable performance during model adaptation. In this work, a concept drift detection strategy followed by a prototype-based adaptation strategy is proposed. Validated through experimental results on a variety of typical non-static data, our solution provides stable and quick adjustments in times of change.

\end{abstract}

\begin{keyword}
Stream Classification, Concept Drift, Robust Soft Learning Vector Quantization, Kolmogorov-Smirnov, RMSprop, Momentum-Based Gradient Descent, Prototype Adaptation
\end{keyword}

\end{frontmatter}



\section{Introduction}
Recent years demonstrated a rapidly increasing amount of data generated by technologies like social media or sensor data. In particular, data is streamed and exceeds the memory and processing capabilities of analyzing systems by far. Hence, streaming algorithms are designed to process data as fast as they arrive, online and without storing large portions of data in the main memory.
In a supervised setting, streams are often affected by a change of underlying class distributions known as concept drift \cite{esann18_raab}. This results in drops of the prediction performance of prior models, making them unusable. The detection and handling of these events is one key area in the field of streaming research. Drifts appear differently through speed, intensity, and frequency, i.e.\, incremental, abrupt, gradual or reoccurring shown in Fig. \ref{fig:drift_types}.

The stability-plasticity dilemma \cite{Gama2014} defines the trade-off between incorporating new knowledge into models (plasticity) and preserve prior knowledge (stability). This prevents stable performance over time because on the edge of a drift, significant efforts going into learning and testing against new distributions.

The main contribution of this work is a concept drift streaming algorithm able to maintain stability during drift while learning new concepts.
The Robust Soft Learning Vector Quantization (RSLVQ) \cite{Seo2003}, recently considered as stream classifier \cite{Straat2018}, is enhanced by a prototype adaptation technique and combined with the Kolmogorov-Smirnov (KS)-Test for concept drift detection and referred as Reactive Robust Soft Learning Vector Quantization (RRSLVQ). The RRSLVQ is tested against standard concept drift stream classifiers on benchmark streams and showed stability during the drift, while quickly learning new concepts. Further, frequent reoccurring concept drift, which has not yet been considered in the literature, is introduced in Sec. \ref{sec:concept_drift} and integrated into the study.

The remaining structure of the paper is given in the following.
We discuss prior work on modifications of LVQ-prototypes and concept drift detectors in Sec. \ref{sec:related_work}.
In Sec. \ref{sec:preliminaries}, stream classification and concept drift are introduced.
The RRSLVQ as the main contribution is discussed in Sec. \ref{sec:rrslvq} and is divided into two parts. The concept drift detector is described in Sec. \ref{sec:concept_drift_detector}. The prototype adaptation strategy is discussed in Sec. \ref{sec:nsm} and its properties are shown in \ref{sec:properties_adaption}. An experimental study in Sec. \ref{sec:experiments} validates the approach.

\section{Related Work}\label{sec:related_work}

Concept drift detectors are trying to detect a change in streams either by monitoring the distribution of the streams or the performance of a classifier with respect to some benchmark, e.g.\, accuracy.
A popular approach for monitoring the prediction accuracy of a classifier is the Adaptive Sliding Window (Adwin)\cite{Bifet2006} and it assumes that if a change in performance is detected, the concept has changed. Adwin identifies changes in distributions using a window $W$ and splits $W$ into two adaptive subwindows and compares underlying statistics. The main window grows as there is no change detected and shrinks if a change between statistics of subwindows is detected. The change is recognized via Hoeffding Bound \cite{Bifet2006}.

Other methods similar to our approach identify drift not only but mainly via statistical tests on distributions of streams.
In \cite{DosReis2016}, two windows are maintained by a randomized search tree, which keeps recent data and the last concept. The KS-Test detects concept drift between windows. The approach simplifies the KS-Test similar to our approach. However, it requires a table of critical values for application. Further research is done in \cite{Salperwyck2015}, where one window stores a snapshot of the concept since last drift and another store a recently surveyed concept.
This detector is supervised and is not able to detect concept drift independent of conditional class probabilities.

Concept drift handling techniques can be roughly divided into active and passive \cite{Losing2017a}. The passive ones have no specific detection strategy, hence continually updating the model without awareness of concept drift. Active adaptation changes a model noticeable. By means of Learning Vector Quantization (LVQ), prototype adaptation or insertion is common \cite{Climer2016a}.

The family of LVQ algorithms is a learning scheme of prototypes representing class regions \cite{Nova2014}. The prototypes have a geometric representation and provide a simple interpretation and classification scheme. Within the learning process, the prototypes are attracted and repelled depending on the class of given samples minimizing the error function. Note that the first version of LVQ is a heuristic-based approach. However, through recent developments in the field, the Generalized-LVQ (GLVQ) \cite{Sato1995} or probabilistic approaches like RSLVQ \cite{Seo2003} are standard \cite{Villmann2017}.

In \cite{Zell1993}, a new prototype is inserted as the mean of a set of misclassified examples. Losing et al. \cite{Losing2015} select one sample from given samples as new prototypes, which minimizes the error. A near-mean technique is proposed by Climer et al. \cite{Climer2016a}. It uses a specific sample as a new prototype, which has the smallest Euclidean distance to the mean of a set of misclassified examples.
Besides great results in the respective domains, they share the problem of assuming the existence of all classes in a given batch of samples or apply asymmetric prototype insertion. Due to the composition of streams, with potentially unbalanced classes \cite{Bifet2015}, this assumption is not guaranteed. In this work, we propose a prototype adaptation strategy that can cope with missing classes.

\section{Preliminaries}\label{sec:preliminaries}
\subsection{Stream Classification}
In supervised classification a stream with potentially infinite length is a sequence $S=\{s_1,\dots,s_t,\dots\}$ of tuples $s_i=\{\mathbf{x_i},y_i\}$, arriving one $s_t$ at time $t$. A stream classifier predicts labels $y_t\in\{1,\dots,C\}$ of unseen data $\mathbf{x}_t \in \mathbb{R}^d$ by prior model $\hat y=h_{t-1}(\mathbf{x}_t)$ and includes this tuple into the model afterwards $h_{t}=learn(h_{t-1},s_t)$. This requires algorithms to predict and learn at any time.

Further, a stream classifier must cope with non-stationary environments, which change through the occurrence of concept drift.
\subsection{Concept Drift}\label{sec:concept_drift}
Concept drift is the change of joint distributions of a set of samples and corresponding labels between two points in time by
\begin{equation}\label{eq:concept_drift}
  \exists \mathbf{X} : p_{t-1}(\mathbf{X},y) \neq  p_t(\mathbf{X},y) .
\end{equation}
There are a variety of different drift types occurring in streaming data and causes the concept to change from one time step to another one. Fig. \ref{fig:drift_types} gives a short overview of drift types \cite{Gama2014}. Every sub-fgure shows a particular drift given as data mean. The shapes mark the dominating concept at a given time step. The vertical axis shows the overall data mean and the transition from one to another concept.
The goal of detectors is to identify any change as soon as possible, as in Eq. \eqref{eq:concept_drift}.
Therefore a detector should monitor the data distribution rather than performance values.

Note that we can rewrite Eq. \eqref{eq:concept_drift} to
\begin{equation}
    p_{t-1}(y|\*X) p_{t-1}(\*X) \neq p_t(y|\*X)p_t(\*X).
\end{equation}
If only the prior distribution $p(\mathbf{X})$ changes for two points in time, then it is called virtual drift. We assume that a change in this prior distribution is always present at real concept drift \cite{Gama2014}. Therefore, we assume with the proposed detector to identify every concept drift through monitoring $p(\*X)$. This is also an essential assumption for most statistical tests since they do not observe the conditional class probability but the prior distribution.

The frequent reoccurring drift is shown in the Fig. \ref{fig:freq_reoccurring_drift} and has not yet been considered in the literature but is particularly impressive. Stream classification algorithms are usually tested on streams with a low number of drifts per setting \cite{Gomes2017}. However, in real-world settings like robotics or autonomous driving, the concept can change very often due to changing external conditions like lighting or weather \cite{Gama.04,Wankhade.2013}. Therefore, we also use stream generators with frequent reoccurring drifts.
\begin{figure}
    \centering     
    \begin{subfigure}[]{.49\linewidth}
    \includegraphics[width=1\textwidth]{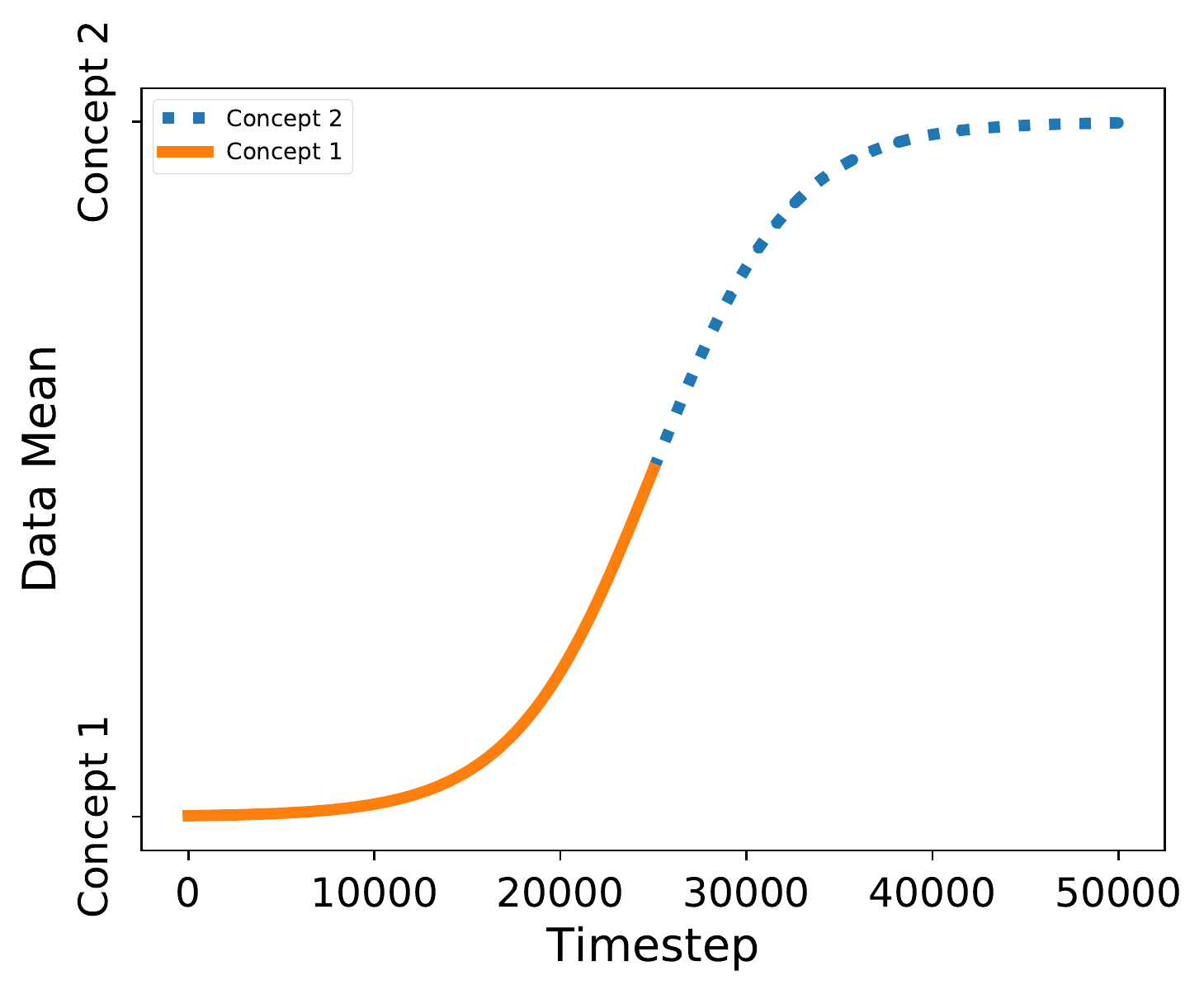}
    \subcaption{Incremental Drift \label{fig:incremental_drift}}
    \end{subfigure}
    \begin{subfigure}[]{.49\linewidth}
    \includegraphics[width=1\textwidth]{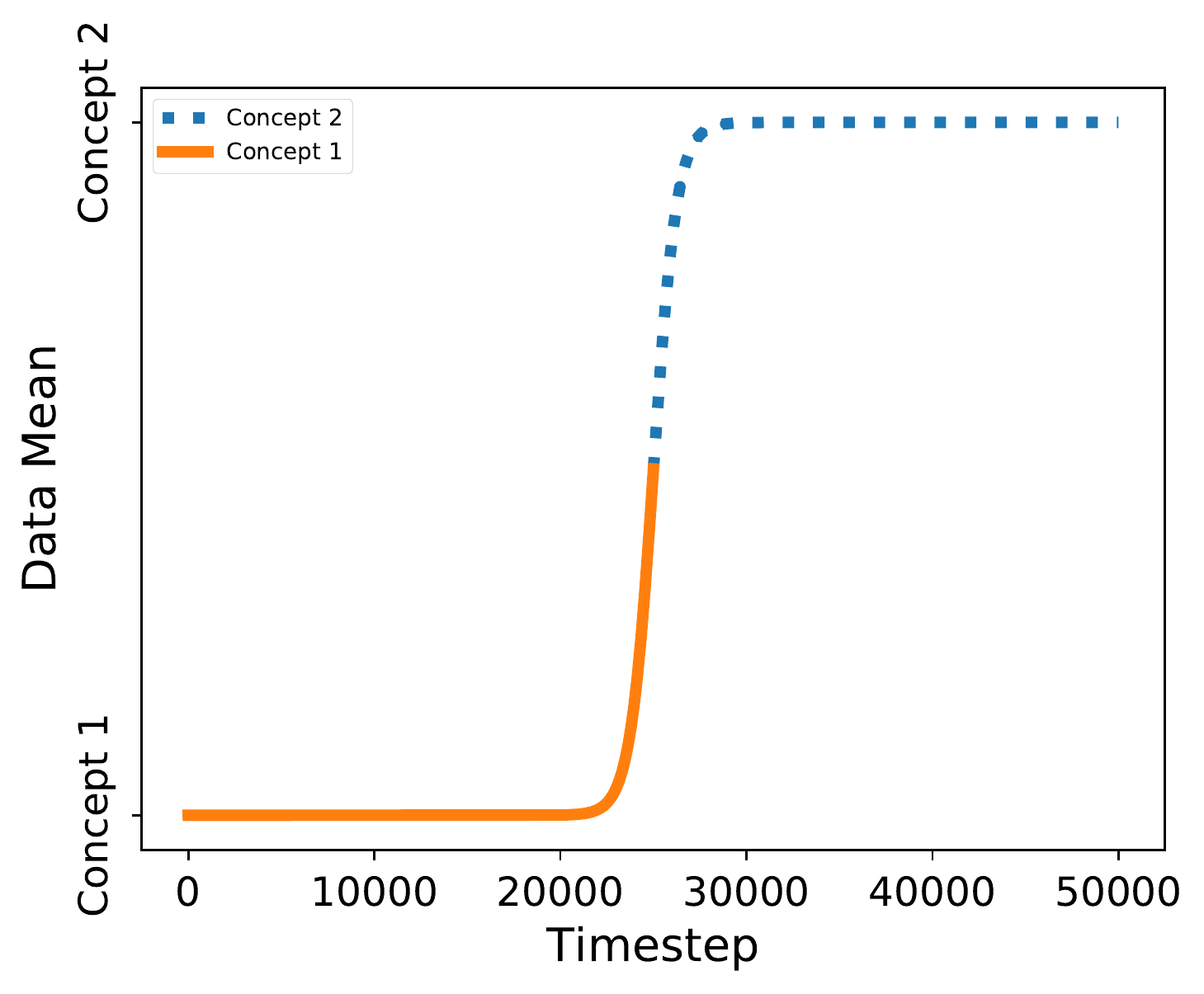}
    \subcaption{Abrupt Drift \label{fig:abrupt_drift}}
    \end{subfigure}
    \begin{subfigure}[]{.49\linewidth}
    \includegraphics[width=1\textwidth]{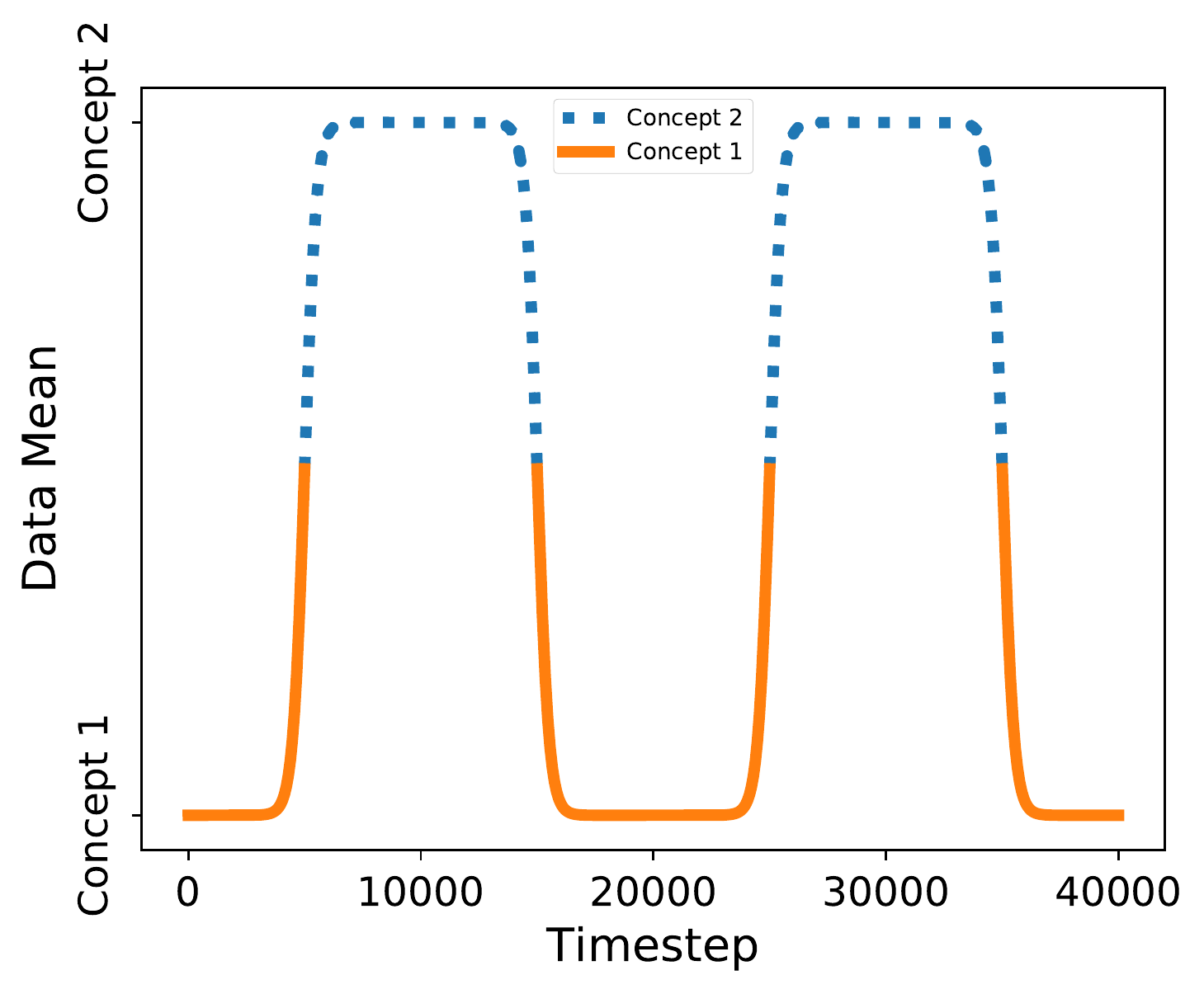}
    \subcaption{Reoccurring Drift }\label{fig:reoccurring_drift}
    \end{subfigure}
    \begin{subfigure}[]{.49\linewidth}
    \includegraphics[width=1\textwidth]{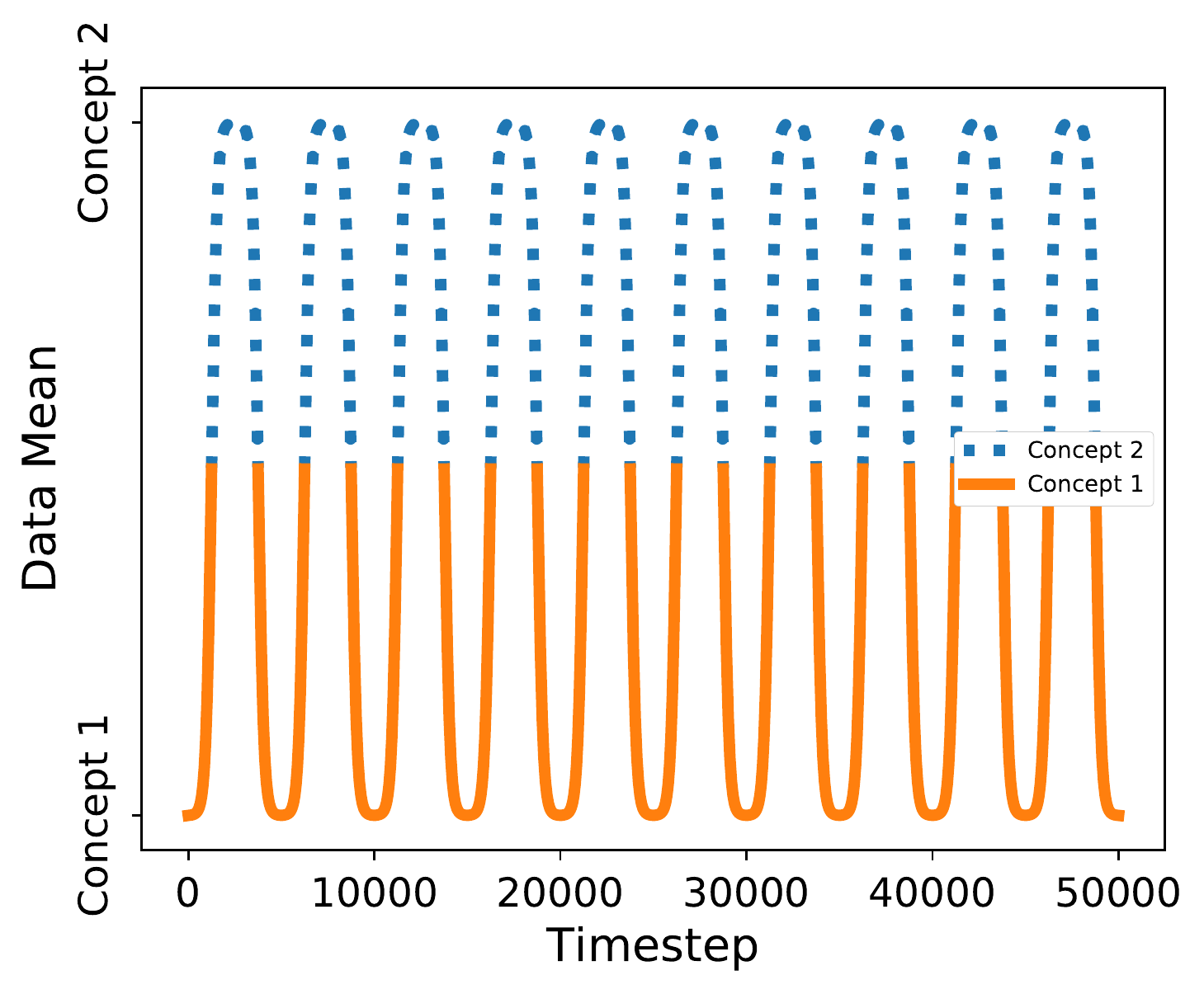}
    \subcaption{Frequent Reoccurring} \label{fig:freq_reoccurring_drift}
    \end{subfigure}
    \begin{subfigure}[]{.49\linewidth}
    \includegraphics[width=1\textwidth]{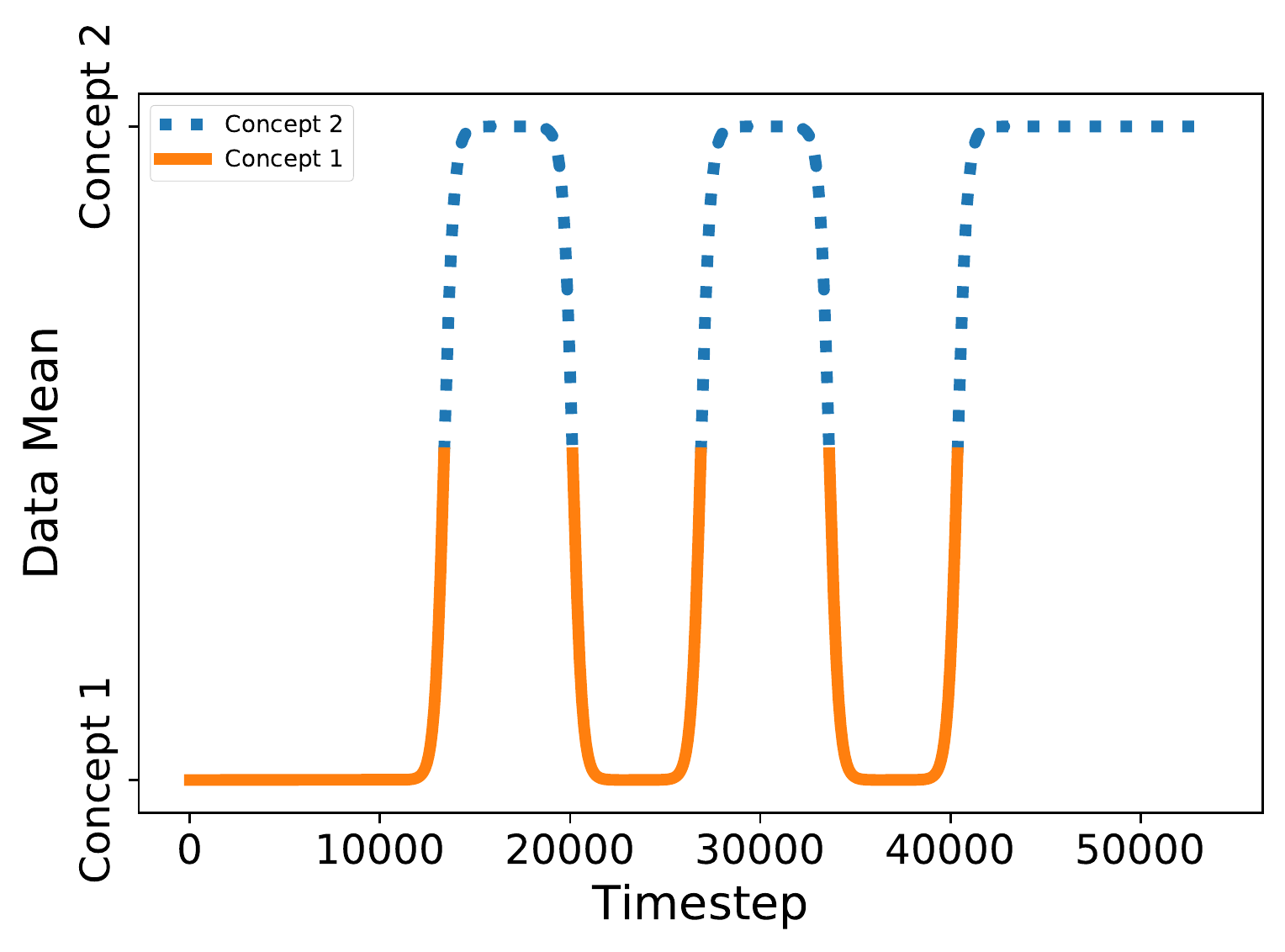}
    \subcaption{Gradual Drift \label{fig:gradal_drift}}
    \end{subfigure}
    \caption{Different types of drifts, one per sub-figure and illustrated as data mean. The colors mark the dominate concept at given time step. The vertical axis shows the data mean and the transition from one to another concept. Given the time axis the speed of the transition is given, which reaches from very slow at incremental to very fast at frequent reoccurring drift. The figures are inspired by \cite{Gama2014}. \label{fig:drift_types}}
\end{figure}
\section{Reactive Robust Soft Learning Vector Quantization}\label{sec:rrslvq}
The section details the main contribution and contains two main parts. The first in Sec. \ref{sec:concept_drift_detector} details the concept drift detector and the second part introduces the RSLVQ (Sec. \ref{sec:rslvq}), the proposed prototype adaptation strategy (Sec. \ref{sec:nsm}) and the properties of the adaptation (Sec. \ref{sec:properties_adaption}).
\subsection{Concept Drift Detection}\label{sec:concept_drift_detector}
Before applying a drift detector, a memory strategy must be defined. The sliding window $\Psi$ keeps $n$ recent points from the stream. It pushes incoming data to the top and removing the oldest one from the bottom. We use the Kolmogorov-Smirnov test as a concept drift detector, which expects two data windows. Therefore, we create two windows out of the sliding window $\Psi$. The first windows $R = \{\mathbf{x}_i \in \Psi\}_{i=n-r+1}^n$ has the most recent data points from $\Psi$. We define most recent as the $r$ newest arrived data points. The second window $W$ is created by sampling uniformly from the non-recent part of $\Psi$ by
\begin{equation}\label{eq:kswin_window}
    W = \{\mathbf{x}_i\in \Psi | i < n-r+1 \land p(\mathbf{x})= \mathcal{U}(\mathbf{x}_i|1, n-r)\}
\end{equation}
with $|W|=|R|= r $ and  $\mathcal{U}(\mathbf{x}_i|1, n-r) = \frac{1}{n-r}$ is the uniform distribution. By this, we do not make any assumptions of distribution but assume to represent the concept of the non-recent data. The memory strategy with the sliding window is summarized in Fig.  \ref{fig:memmory_strategy}. For subsequent adaptation steps, we also store the label $y_i$ to a given sample $\*x_i$.

\begin{figure}[]
    \centering
    \includegraphics[width=1\textwidth]{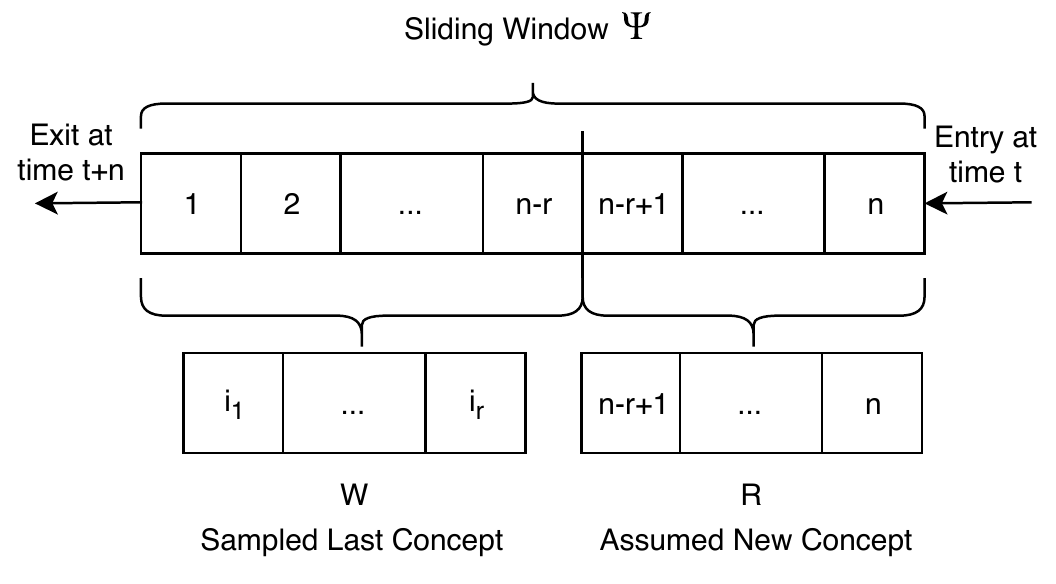}
    \caption{The memory strategy within the concept drift detector. It stores $n$ samples in a sliding window. At every time step, $r$ samples are picked uniform and are compared against the newest $r$ samples of the window.}
    \label{fig:memmory_strategy}
\end{figure}

By using the above scheme, there is a limitation in concept drift detection regarding some streams. The worst case is to take a sample from a recurring stream at certain intervals and always get the same concept of the changing stream. Thus the distribution seems to be the same from window $R$ and the samples $W$, because the detector always picks at the same rate of concept change. However, there is clearly a concept drift ongoing. This is due to the large window and the Uniform picking. Other probability distributions for the sampling scheme should lead to a different result, but with familiar problems and some concept drifts will be missed, because the sampling procedure is an approximation of the window $\Psi$. We choose the Uniform distribution because it has no parameters and is fair in sampling from unknown distributions.

The Kolmogorov-Smirnov \cite{Lopes2011} test is a non-parametric test accepting one-dimensional data without any assumptions of underlying distributions. It compares the absolute distance $dist_{w,r}$ between two empirical cumulative distributions $F_W$ and $F_R$ by
\begin{equation}\label{eq:empirical_distance}
    dist_{w,r} =\sup_{x} |F_{W}(x) - F_{R}(x)|,
\end{equation}
where $F_{(\cdot)}(x)=\frac{1}{n}\sum_{j=1}^{n} I_{[-\infty,x]}(X_j)$ and $I_{[-\infty,x]}(X_j)$ is an indicator function which is one if $X_i \geq x $ and zero otherwise. Note that $sup(x)$ is the smallest required $x$ for a condition to be valid. If this lower bound of maximum distance, i.e.\, $dist_{w,r}$, is greater than the test statistic, then the null hypothesis is rejected, with significance level $\alpha$. For two subwindows $W,R$ with the same size, the test is reduced to
\begin{eqnarray}\label{eq:kstest}
   dist_{w,r}  >  c(\alpha) \sqrt{\frac{n+r}{nr}}= \sqrt{-\frac{1}{2}\ln\alpha}\sqrt{\frac{n+r}{nr}} \stackrel{\text{(n=r)}}{=} \sqrt{-\frac{\ln\alpha}{r}},
\end{eqnarray}
where $\alpha$ is the confidence level, $r$ is the size of the window $R$ and $n$ is the size of the window $w$. $c(\alpha)$ is the critical value of the test with respect to the confidence level $\alpha$, which can either be looked up for standard test sizes or computed by $\sqrt{-\frac{1}{2}\ln\alpha}$.
Due to the restriction to one-dimensional distributions, the test in Eq. \eqref{eq:kstest} is applied to all dimensions, therefore at any point $t$, $d$ tests must be done due to $\mathbb{R}^d$. With this extension and based on Eq. \eqref{eq:kstest}, the following can be stated:
\begin{lemma}\label{le:ks_bound}
  Given $\mathbf{X}_{t}$ and $\mathbf{X}_{t-1}$ with $\mathbf{X}_k=\{\mathbf{x}_j\}_{j=1}^{n} \in \mathbb{R}^d$, $p(\mathbf{X})= \prod\limits_{i=1}^{d} p(\mathbf{x}^{(i)})$, $\forall \mathbf{x}^{(i)}$  $\sim i.i.d$ and their cumulative distributions $F_{t}$,$F_{t-1}$, KS-Test detects any change between $p(\mathbf{X}_{t})$ and $p(\mathbf{X}_{t-1})$, i.e.\, $\exists \mathbf{x}^{(i)} : p_{t}(\mathbf{x}^{(i)}) \neq p_{t-1}(\mathbf{x}^{(i)})$,  with probability $1-\alpha$, if the difference in empirical distribution is at least $\sqrt{\frac{-ln \alpha}{r}}$.
\end{lemma}
To implement Lemma \ref{le:ks_bound}, we set $\mathbf{X}_{t}=R$ and $\mathbf{X}_{t-1}=W$.
For a reasonable choice of window size $r$, the influence of changing the window size and the confidence level of the KS-Test is demonstrated in Fig. \ref{fig:parameter_sensitivity}. The left figure shows the influence of the window size and the right figure, the confidence level on the required distance. By decreasing the confidence level $\alpha$, $dist$ increases, and with increasing window size $r$, $dist$ decreases. Hence, they behave competitively w.r.t. $dist$ value. In a streaming context with potentially unlimited data, the KS-Test will be too sensitive, given a large window size,  with many false positives. This motivates the choice of a relatively small window size $r=30$, but still being statistically valid at the same time.
\begin{figure}[t!]
    \centering     
    \begin{subfigure}[]{.49\linewidth}
    \includegraphics[width=1\textwidth]{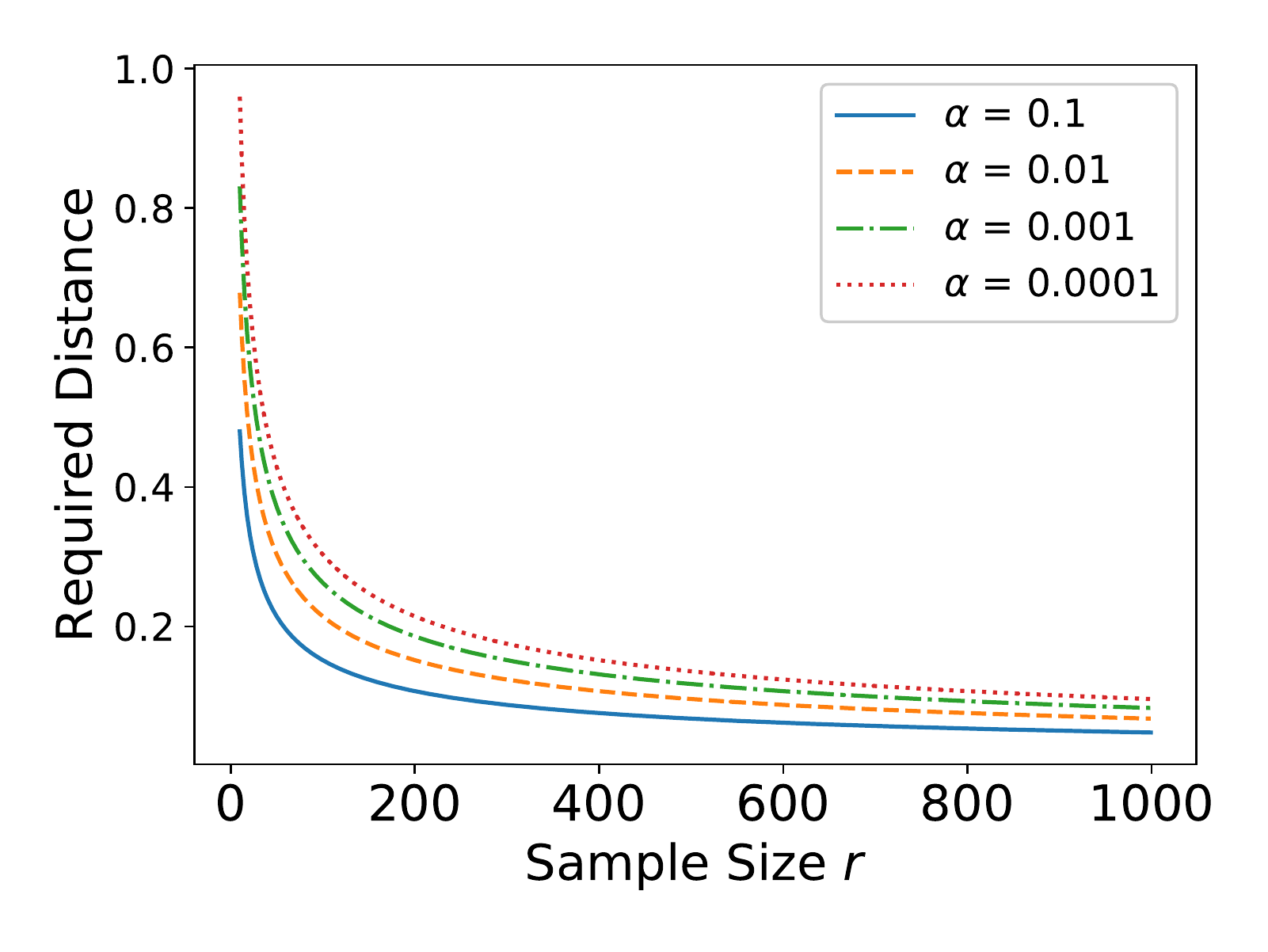}
    \subcaption{Window Size (Log scale)\label{fig:window_size}}

    \end{subfigure}
    \begin{subfigure}[]{.49\linewidth}
    \includegraphics[width=1\textwidth]{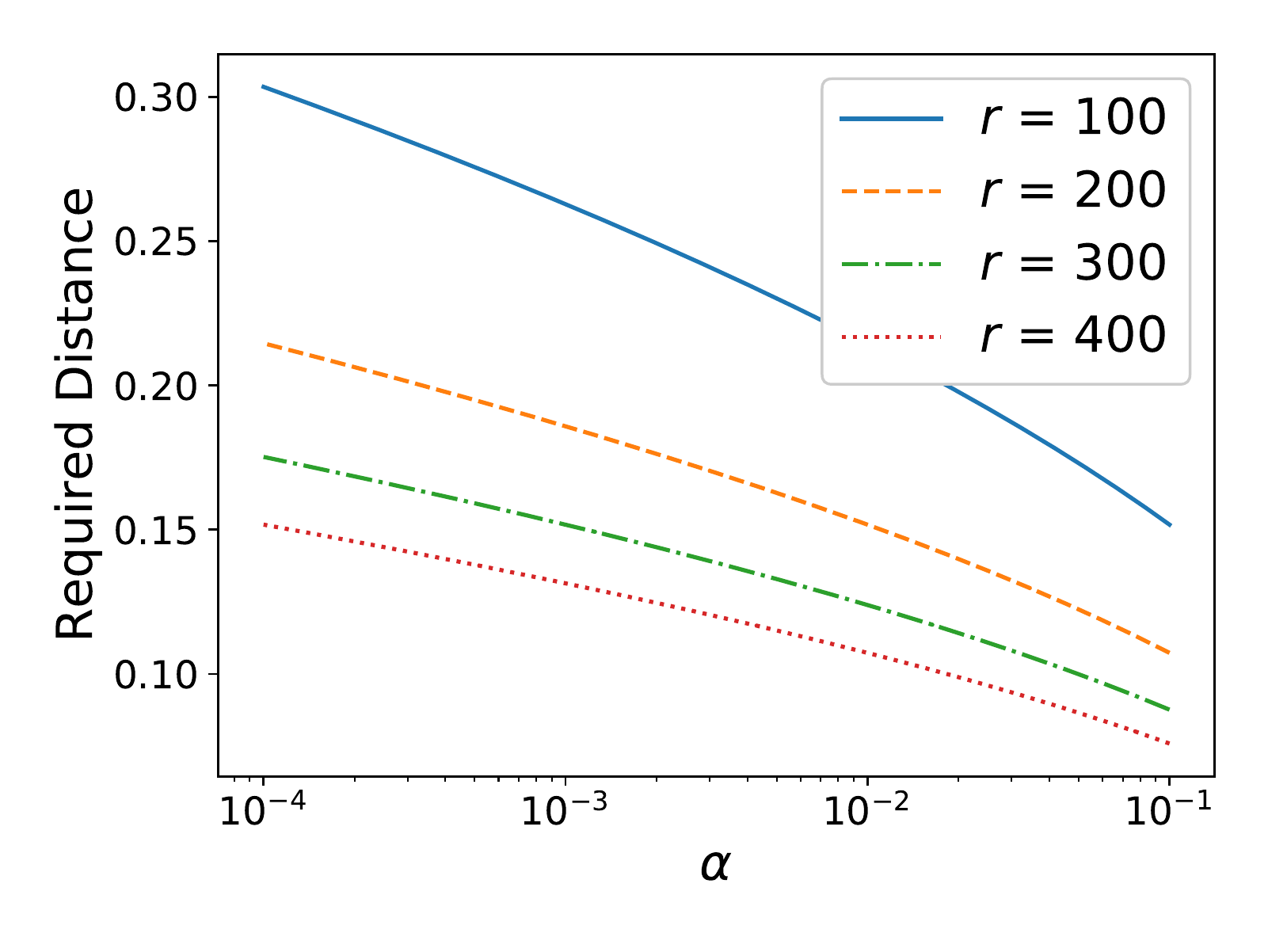}
    \subcaption{Confidence Level (Log Scale) \label{fig:confident_score}}
    \end{subfigure}
    \caption{Plot of the parameter sensitivity to the required distance of the KS-Test in log scale on the x-axes. The left figure shows the sample size and the right figure, shows the confidence level against the required distance needed for a concept drift detection. Both figures show sensitivity towards the window size and especially with a tiny window the test becomes more insensitive to small differences in distribution. \label{fig:parameter_sensitivity}}
\end{figure}

However, when it comes to many statistical tests, the problem with multiple hypothesis testing arises, with the consequence of false positives due to random chance.
By applying the Bonferroni-Dunn correction \cite{Abdi2007}, we reduce this effect by obtaining a new alpha with $\alpha^* = \lfloor \frac{\alpha}{r}\rfloor = \frac{0.05}{60} =0.0001$.
Because $\alpha^*$ affects Eq. \eqref{eq:kstest}, a small alpha increases the required distance.

Both, a small window size and the corrected alpha should avoid false positives. However, both actions do not entirely avoid false signals, but making the KS-Test more insensitive and combined with the proposed prototype adaptation strategy false positives are not critical. In the following, the KS-Test detector is called Kswin.

\subsection{Robust Soft Learning Vector Quantization}\label{sec:rslvq}
The Robust Soft Learning Vector Quantization \cite{Seo2003} is a probabilistic prototype based classification algorithm and capable of online learning.
Given a labeled dataset $ \mathbf{D} = \{ ( \mathbf{x_i},y_i) \in \mathbb{R}^d \times \{1,\dots,C\}\}_{i=1}^n $  as classification task. The RSLVQ assumes that $\mathbf{D}$ can be represented as class-dependent Gaussian mixture model and approximates this mixture by a set of $m$ prototypes $\Theta= \{ (\boldsymbol{\theta}_j,y_j) \in \mathbb{R}^d \times \{1,\dots,C\}\}_{j=1}^m$, where each prototype represents a multi-variate Gaussian model, i.e.\, $\mathcal{N}(\boldsymbol{\theta}_j,\Sigma)$ with $\Sigma = \*I \sigma$ and $\*I$ as identity matrix. Further, $\b\theta_j$ is a d-dimensional prototype representing the mean of a Gaussian mixture component with a variance of $\sigma$, which is assumed to be the same for every component.
The goal of RSLVQ algorithms is to learn prototypes representing the class-dependent distribution, i.e.\, a corresponding class sample $\mathbf{x}_i$ should be mapped to the correct class or Gaussian mixture based on the highest probability.

The RSLVQ \cite{Seo2003} algorithm minimizes the objective function
\begin{equation}\label{eq:objective_rslvq}
  \mathcal{L}  = \frac{1}{n} \sum\limits_{i=1}^n ls(\mathbf{x}_i,{y_i}|\Theta)\>\>\>\text{with} \>\>\>   ls(\mathbf{x}_i,{y_i}|\Theta)=  \frac{p(\mathbf{x}_i,\Bar{y_i}|\Theta)}{p(\mathbf{x}_i|\Theta)}.
\end{equation}
Where $p(\mathbf{x}_i,\Bar{y_i}|\Theta)$ is the probability density function that $\mathbf{x}_i$ is generated by the mixture model of any of the different classes and $p(\mathbf{x}_i|\Theta)$ is the overall probability density function of $\mathbf{x}_i$ given $\Theta$. In other words, $\Bar{y_i}$ is every label which is not the ground truth label $y_i$ of $x_i$. At time step $t$, these probabilities are computed by
\begin{equation}\label{eq:rslvq_lst}
    ls_t(\mathbf{x}_t,{y_t}|\Theta) = \frac{p(\mathbf{x}_t,\Bar{y_t}|\Theta)}{p(\mathbf{x}_t|\Theta)} = \sum_{j:c_j \neq y_t} P(j|\*x_t).
\end{equation}
Where $j$ is the $j$-th prototype in $\Theta$ and $c_j$ is the corresponding label. $P(j|\mathbf{x})$ is the probability that $\mathbf{x}$ is generated by the component $j$ with
\begin{equation}\label{eq:rlsq_probs}
    P(j|\*x) = \frac{p(\*x|j)P(j)}{p(\*x)} = \frac{exp\Big(-\frac{\norm{\*x - \b\theta^2_j}}{2\sigma^2} \Big)}{\sum_{k:c_k \neq c_j} exp\Big(-\frac{\norm{\*x - \b\theta^2_k}}{2\sigma^2} \Big)}.
\end{equation}
Based on this, the gradient \cite{Seo2003} for the prototype update step at time $t$ is computed by
\begin{equation}\label{eq:rslve_gradients}
    g_t = \begin{cases}
        - P(j|\*x_t)ls_t(\*x_t - \b\theta_t),   & if\>\>c_j = y_t,\\
        P(j|\*x_t)(1-ls_t)(\*x_t - \b\theta_t), & if\>\>c_j \neq y_t. \\
\end{cases}
\end{equation}
Note that $ls_t$ abbreviated for Eq. \eqref{eq:rslvq_lst}. The normalization term in Eq. \eqref{eq:objective_rslvq} is not computed at the update step and, therefore, the RSLVQ is feasible for potentially infinite streams.

The prototypes in each update step will be optimized with a momentum-based gradient technique designed for RSLVQ \cite{Heusinger} given with
\begin{equation}\label{eq:rslvq_update}
   \b\theta_{t+1} = \b\theta_{t} - \Delta\b\theta_{t} = \b\theta_t - \frac{\sqrt{E[\Delta\theta^2]_{t} + \epsilon}}{\sqrt{E[g^2]_{t} + \epsilon}}g_t.
\end{equation}
Where $\epsilon$ is a small positive value for numerical stability. The $E[g^2]_t$ are the stored past squared gradients as running mean
\begin{equation}\label{eq:update_gradients}
    E[g^2]_t = \gamma E[g^2]_{t-1} + (1 - \gamma)g_t^2
\end{equation}
and $E[\Delta\theta^2]$ are the past squared parameter updates as running mean
\begin{equation}\label{eq:update_parameters}
    E[\Delta\theta^2]_t = \gamma E[\Delta \theta^2]_{t-1} + (1 - \gamma) \Delta \theta_t^2.
\end{equation}
Both equations are controlled by the decay-factor $\gamma$ controlling the relevance of previous updates and the current ones.
As pointed out in \cite{Biehl.19}, a momentum-based gradient descent is a reliable strategy to handle concept drift with LVQ classifiers.
In a streaming setting at every time step, every prototype will be updated by Eq. \eqref{eq:rslvq_update} with a given tuple $s_i=\{\*x_i,y_i\}$, as shown in pseudocode \ref{PseudoCodeRSLVQ}.
The RSLVQ predicts a given sample point $\*x$ by selecting the nearest prototype $\b\theta_q$ according to the Gaussian kernel
\begin{equation}\label{eq:rslvq_prediction}
    q = \argmin_i exp \bigg(-\frac{\norm{\*x - \b\theta^2}}{2\sigma^2} \bigg),
\end{equation}
and assigning the corresponding label of $\b\theta_q$ to $\*x$. The $\sigma$ is the width of the Gaussian kernel and together with the number of prototypes, are the only tuneable parameters in the RSLVQ.
For a more comprehensive derivation of RSLVQ with momentum SGD see \cite{Seo2003,Heusinger}.

\subsection{Prototype Adaptation Strategy}\label{sec:nsm}
In this work, the adaptation strategy adapts actively to abrupt or gradual concept drift. We allow missing classes in $R$ since streams have unbalanced classes and class-wise adaptation is usually infeasible.
If a drift is detected as in Sec. \ref{sec:concept_drift_detector}, window $R$ represents a new context and $\Theta$ represents the last context. According to Kswin, the current set of prototypes $\Theta$ has approximated an out-dated distribution.

Hence, we propose the following two-step adaptation strategy.
The first step is to replace the complete set $\Theta$ of prototypes by $m$ new prototypes. Note that the classes remain the same.
A good starting point for replacement by means of approximating an unknown mixture model by Gaussian mixture is the mean of points\cite{Seo2003}, regardless of class affiliation. Therefore, the new prototypes become
\begin{equation}\label{eq:replacement}
     \boldsymbol{\theta}_{i}^{(new)} = \frac{1}{r}\sum_{\mathbf{x} \in R} \mathbf{x},\>\>\>\forall i=1,\dots,m.
\end{equation}
The next step is to train all prototypes as in Eq. \ref{eq:rslvq_update} on all $\{\*x_i,y_i\} \in R$.
This modifies prototypes and draws them to the same class points and pushes different class prototypes away.
In the case of missing classes, the algorithm pushes the prototypes away from points with a different class.

The process is shown in pseudocode \ref{PseudoCodeRSLVQ} in the \textit{State 3 - Learning Rate}. Further, we visualized the adaptation in Fig. \ref{fig:InsertSteps}. It shows the adaptation of two prototypes with different classes as mean in \ref{fig:replacement} and the optimization afterward in \ref{fig:optimization} in a two-dimensional feature space. Data is generated with MIXED generator stream from Scikit-Multiflow \cite{skmultiflow}.
Further, we show in Sec. \ref{sec:properties_adaption} that adaptation on $R$ leads to a lower error, if the distributions are different enough, rather then taking no action.
Therefore, false positives as described in Sec. \ref{sec:concept_drift_detector}, are not critical because every replacement always leads to lower error and does no harm.
\begin{figure}[t!]
    \centering     
    \begin{subfigure}[]{.49\linewidth}
    \includegraphics[width=1\textwidth]{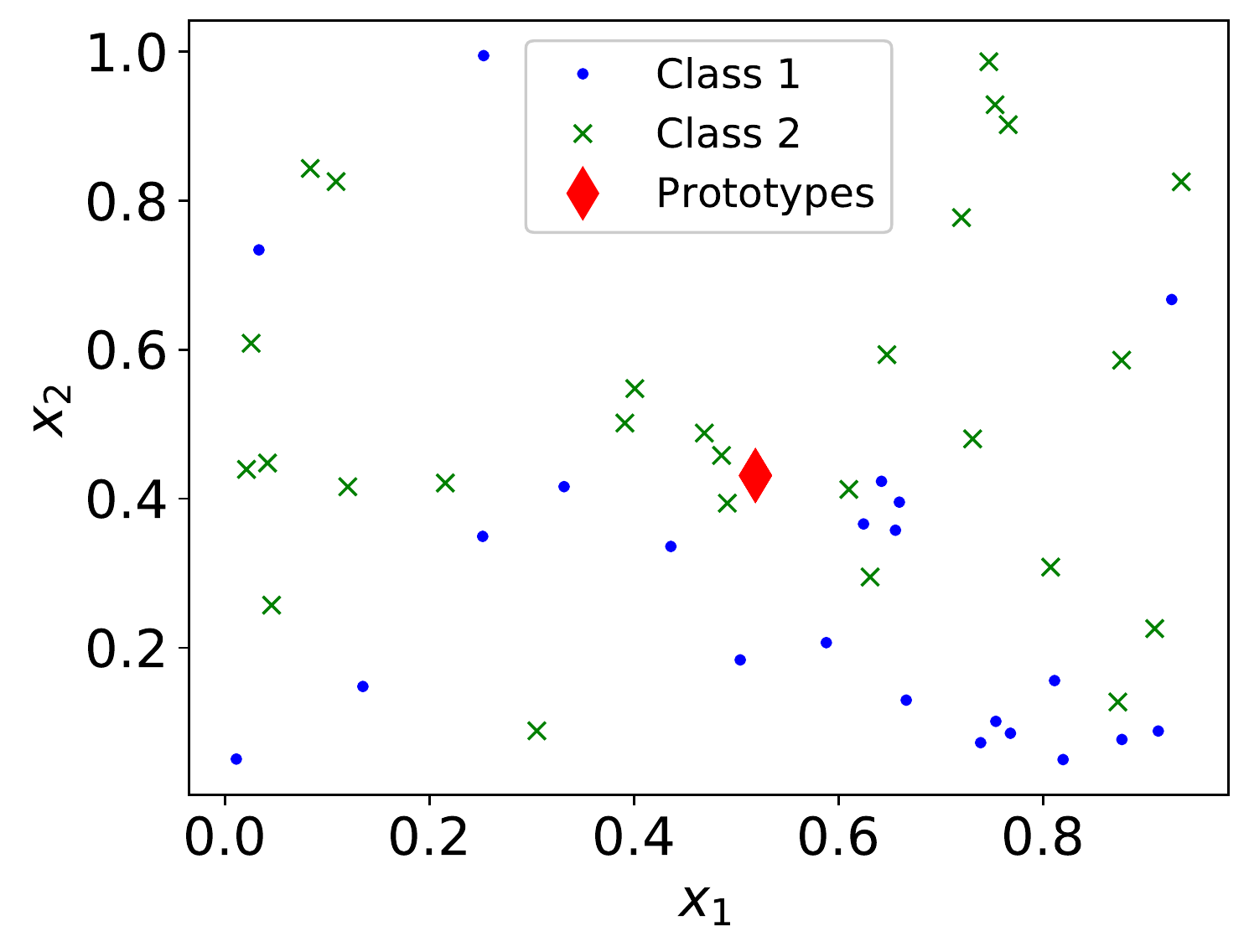}
    \subcaption{Insertion Step    \label{fig:replacement}}

    \end{subfigure}
    \begin{subfigure}[]{.49\linewidth}
    \includegraphics[width=1\textwidth]{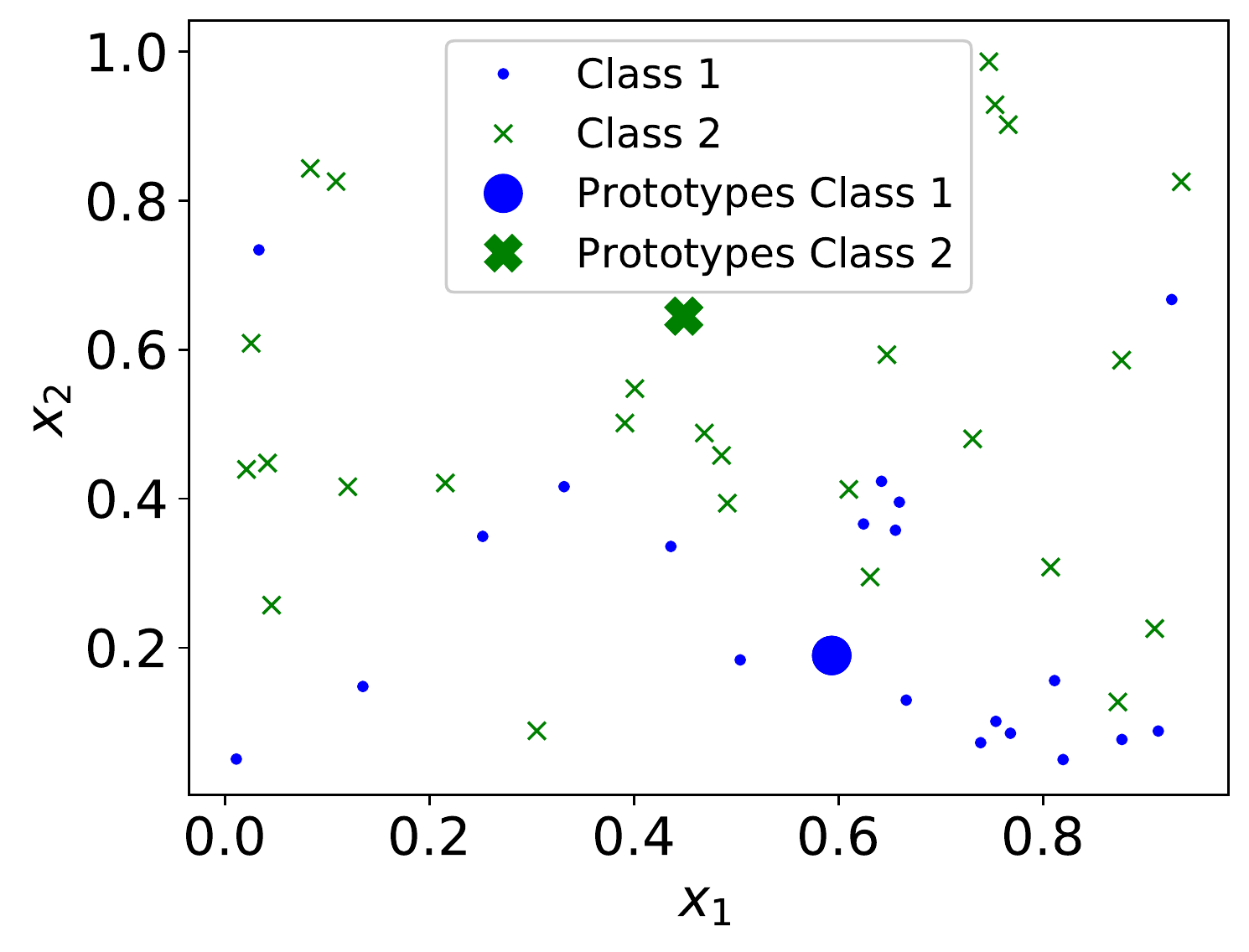}

    \subcaption{Optimization Step    \label{fig:optimization}}
    \end{subfigure}
    \caption{Process of  prototype adaptation strategy on given window $R$ with $|R|=50$ for more details. First, insertion of $m$ prototypes as mean over all samples in the left figure, marked as a red diamond. Second after optimization, where green/blue crosses/dots are representing class prototypes. \label{fig:InsertSteps}}
\end{figure}
\begin{algorithm}
    \caption{Reactive Robust Soft Learning Vector Quantization ({RRSLVQ})}\label{PseudoCodeRSLVQ}
    \begin{algorithmic}[1]
        \phasebegin{Initialization\vspace{-5pt}}
        \Require $m$ as number of prototypes; $\sigma$ as width of Gaussian kernel; $\gamma$ as decay-factor; $r$ as sampling size; $\alpha$ as confidence-level.
        \Ensure $h_0(\*x;\Theta,\gamma,\alpha$,$r$) as initialized model
        \State Create prototypes $\Theta = \{\b\theta_j\}_{j=1}^m \sim \mathcal{N}(\*0,\Sigma=\*I\sigma)$
        \phase{ Prediction}
        \Require $h_{t-1}(\*x)$ as RRSLVQ model; $\*x_t\in\mathbb{R}^d$ sample point arrived at time $t$.
        \Ensure Predicted label $\hat{y}_t$ of sample $\*x_t$.
        \State Compute nearest prototype $\b\theta_q$ to $\*x_t$ (Eq. \eqref{eq:rslvq_prediction})
        \State Assign predicted label $\hat{y}$ based on the neared prototype $\b\theta_q$
        \phase{Learning}
        \Require $h_{t-1}(\*x)$ as RRSLVQ model from previous time step; $\{\*x_t,y_t\}$ as labeled data at time step t.
        \Ensure $h_{t}(\*x)$ as updated model.
        \State Update sliding window (According to Fig. \ref{fig:memmory_strategy})
        \State Create sampling window $W$ and $R$(Eq. \ref{eq:kswin_window} and Fig. \ref{fig:memmory_strategy})
        \State Compute $d$ distribution differences $dist_{w,r}$ between $R,W \in R^d$ (Eq. \eqref{eq:empirical_distance})
        \If{ $\exists d_i > \sqrt{-\frac{\ln\alpha}{r}}$ (CD-Detection Eq. \eqref{eq:kstest})}
            \For{$\forall \*x_i,y_i \in R$}
                \State Step 12 to 16
            \EndFor
        \EndIf
        \For{$\b\theta_j : j=1 \to m$}
            \State Compute posterior probabilities (Eq. \eqref{eq:rlsq_probs})
            \State Compute gradient $g_t$ for $\b\theta_j$ (Eq. \eqref{eq:rslve_gradients})
            \State Compute past gradient and parameter updates (Eq. \eqref{eq:update_gradients} and \eqref{eq:update_parameters})
            \State Update prototype $\b\theta_{j+1} = \b\theta_j - \Delta\b\theta_{t}$ (Eq. \eqref{eq:rslvq_update})
        \EndFor
        \end{algorithmic}
\end{algorithm}
\subsection{Stability during Concept Drift}\label{sec:properties_adaption}
The stability during drift will be shown in the following and we start with the introduction of notation and background.
Given a two-class classifier $h : \mathcal{X} \to \{-1,1\}$  on a given sample $\*x \in \mathcal{X} \subseteq \mathbb{R}^d$. The loss of a classifier $l(h(x)) \in \mathbb{R}^+$ is measuring the performance of a given set of labeled data sampled from the concept $C = p(x,y) = p(y|x) p(x)$.
The expected risk $R(h(x))$ and the empirical risk $\hat{R}(h(x))$ of the classifier $h(x)$ \cite{Cornuejols2010} are defined as
\begin{equation}\label{eq:risk}
    \begin{aligned}
        R(h(x)) &= \mathbb{E}[l(h(x),y)] \\
        &= \int l(h(x),y) p(y|x)p(x) dxdy \leq \hat{R}(x) = \frac{1}{n} \sum_{i=1}^n l(h(x),y),
    \end{aligned}
\end{equation}
while computing the empirical risk with $n$ samples.
\begin{lemma}\label{lemma:full}
 Given two classifiers $h_1(x)$ and $h_2(x)$ with loss $l(h_{(\cdot)}(x))$ and trained on two different concepts $C_1$ and $C_2$ with $c$ classes each and corresponding risks $R_{(\cdot)} (h_{(\cdot)})$. The empirical risk $R_2(h_1(x))$ on $C_2$ can be bounded by
  \begin{equation}
         \hat{R}_2(h_2(x)) \leq 1 - \frac{1}{c}\leq \delta R_2(h_1(x)) \leq \hat{R}_2(h_1(x)),
  \end{equation}
  under the assumption that there is a linear relationship $\delta =  \frac{p_2(x)}{p_1(x)}$ with $\delta \in \mathbb{R}^+$ between the prior of the concepts.
\end{lemma}
\begin{proof}
First consider the first classifier on the concept one, where it is trained on with the risk
\begin{equation}\label{eq:risk_concept_1}
    \begin{aligned}
        R_1(h_1(x)) &= \mathbb{E}[l(h_1(x),y)] = \int l(h_1(x),y) p_1(y|x)p_1(x) dxdy,
    \end{aligned}
\end{equation}
where $R_1$ is the risk corresponding to concept $C_1$. Applying $h_1(x)$ to concept two leads to the expected risk
\begin{equation}\label{eq:risk_h1_r1}
    \begin{aligned}
        R_2(h_1(x)) =  \int l(h_1(x),y) p_2(y|x)p_2(x) dxdy.
    \end{aligned}
\end{equation}
Similar as in \cite{Cornuejols2010}, we assume the relationship between the two concepts by $p_2(x) = \delta p_1(x)$, e.g.\, two variants of Gaussian distributions, then we can rewrite the risk to
\begin{equation}\label{eq:risk_applied}
    \begin{aligned}
        R_2(h_1(x)) &= \int l(h_1(x),y) p_2(y|x) {\delta} p_1(x) dxdy \\
        &= \delta \int l(h_1(x),y) p_2(y|x)  p_1(x) dxdy.
    \end{aligned}
\end{equation}
The expected risk $R_2(h_2(x))$ is analog to Eq. \eqref{eq:risk_h1_r1}. If the difference in prior distribution between the two concepts is large enough, then $\delta$ becomes large. Hence, for large $\delta$, the expected risk of $R_2(h_1(x))$ becomes larger than random guessing, which is $1-\frac{1}{c}$ for $c$ classes. Because $h_2(x)$ is trained on concept two, we expect that the empirical risk $\hat{R}_2(h_2(x))$ in the worst-case is random and therefore
\begin{equation}
        \hat{R}_2(h_2(x)) \leq 1-\frac{1}{c}\leq \hat{R}_2(h_1(x)).
\end{equation}
\end{proof}
This can be implemented as $h_1(x)$ being the RSLVQ model at time $t-1$ and $h_2(x)=g(R)$ at time $t$ after the adaptation procedure. Hence, for large enough differences between windows, which the Kswin detects given a small $r$ and $\alpha$, the performance of the adaptation yields a smaller loss in comparison to the model from $t-1$. This also means that the classification result is not negatively affected even if Kswin detects a change in prior distribution within a concept, which is not real concept drift. Therefore, the false positives of Kswin do not harm. The effect is demonstrated in Fig. \ref{Fig:PerformanceRC}, showing that $\delta$ is always sufficiently large so that $g(R)$ achieves better performance on the Mixed concept drift stream.

\subsection{Time and Memory Complexity}\label{sec:memory_time_complexity}
The RRSLVQ optimized via online momentum gradient descent has the time complexity of $\mathcal{O}(m)$ at given time $t$ for $m$ prototypes without concept drift.
The concept drift handling has the complexity $\mathcal{O}(r\cdot m)$. The KS-Test can be implemented in log-linear time \cite{DosReis2016}.
The demand for memory depends on the number of prototypes $m$ and sliding window size $n$. Therefore, the model needs at maximum $m\times d$ and $n \times d$ as real float numbers in memory, which accumulates in the complexity of $\mathcal{O}(d\cdot(m+n)) = \mathcal{O}(d\cdot(m+300))$ with $n=300$. For using RRSLVQ in embedded systems with restrictive memory, we follow the \textit{any memory framework} by \cite[p. 6]{Gama2014} and we suggest setting the number of prototypes plus the window length to a maximum of $d\cdot(m+300)< k$ with $k$ as maximum float memory storage capacity.

\section{Experiments}\label{sec:experiments}
The experimental section of the paper consists of three parts. First, we provide the study design, the performance metrics and a data set overview with their properties. The second part analyses the performance of Kswin as an independent approach compared to other concept drift detectors. The last part shows the performance of the RSLVQ paired with the Kswin detector.
\subsection{Study Design}
We use six stream generators (synthetic data) and six real-world datasets. We follow the study design of \cite{Bifet2015}, but additional, we simulate frequent gradual/abrupt reoccurring drift with streams having abrupt or gradual drift. This type of drift stream introduces reoccurring drift with a certain frequency. The frequency is specified in section \ref{sec:comparison_cd} and section \ref{sec:comparison_cls} respectively. An overview of the streams is given in Tab. \ref{tab:used-streams}. The synthetic and the real-world datasets are described in detail in \ref{sec:dataset_description}.

In a stream setting, prediction performance is measured with interleaved test-then-train accuracy \cite{Losing2017a,Bifet2015}. It is the moving average accuracy including the current sequence of data at time $t$ by
\begin{equation}
    A(S) = \frac{1}{t} \sum_{i=1}^t 1(h_{i-1}(\mathbf{x}_i) = y_i).
\end{equation}
Where $1(\cdot)$ is an indicator function and is one for a correct classification and zero otherwise. For the evaluation, the length of the streams is fixed to $t_{max}=1,000,000$. At $t_{max}$, the interleaved test-then-train accuracy becomes the overall mean accuracy. We use this and the kappa-statistics for the subsequent performance tables.
A stream classifier within this setting first predicts the given sequence and then learns with it. This evaluates the ability to predict and learn anytime over a long period \cite{Losing2017a}.

\begin{table}[ht]
    \centering
    \resizebox{\textwidth}{!}{%
    \begin{tabular}{llclcc}
        \hline
        \textbf{Data set} & \textbf{\# Instances} & \textbf{\# Features} & \textbf{Type} & \textbf{Drift} & \textbf{\# Classes}\\
        \hline
        SEA$_a$ & 1,000,000 & 3 & Synthetic & A & 2\\
        SEA$_g$ & 1,000,000 & 3 & Synthetic & G & 2\\
        MIXED$_a$ & 1,000,000 & 4 & Synthetic & A & 2\\
        MIXED$_g$ & 1,000,000 & 4 & Synthetic & G & 2\\
        RTG & 1,000,000 & 10 & Synthetic & N & 2\\
        RBF$_m$ & 1,000,000 & 10 & Synthetic & I$_m$ & 5\\
        RBF$_f$ & 1,000,000 & 10 & Synthetic & I$_f$ & 5\\
        HYPER & 1,000,000 & 10 & Synthetic & I$_f$ & 2\\
        \hline
        POKER & 829,201 & 11 & Real & - & 10\\
        GMSC & 120,269 & 11 & Real & - & 2\\
        ELEC & 45,312 & 8 & Real & - & 2\\
        COV & 581,012 & 27 & Real & - & 7\\
        AIR & 539,383 & 8 & Real & - & 2\\
        SQRE & 200,000 & 2 & Real & - & 4\\
        \hline
    \end{tabular}}
    \caption[Configuration of the data sets]{Configuration of the data sets (A: Abrupt Drift, G: Gradual Drift, $\mathrm{I}_m$: Moderate Incremental Drift, $\mathrm{I}_f$: Fast Incremental Drift and N: No Drift)}
    \label{tab:used-streams}
\end{table}

\subsection{Comparison of Concept Drift Detectors}\label{sec:comparison_cd}
The Kswin detector is tested against other commonly used detectors, i.e\ Adwin \cite{Bifet2006}, Drift-Detection Method (DDM) \cite{Baena006} and Early Drift Detection Method (EDDM) \cite{Gama2004}. The parameters for Kswin are $(r = 30,n=300)$ and $\alpha=0.0001$. The Adwin parameter is $\alpha=0.002$. The DDM algorithm has a minimum number of test size, which is set to $30$ and a detection threshold set to $3$. The window size of EDDM is also set to $30$ with a detection threshold of $\beta = 0.95$.

Each tested stream has 100,000 time steps with a batch size of ten per time step and in summary, there are one million samples per stream.
All time steps of occurring drifts are compared against the predicted time steps of the detectors, which are either zero for no concept drift and one if drift is detected. Summarizing, there are ten standard concept drift streams with one concept drift per stream and six frequent reoccurring concept drift streams with 99,999 occurrences of concept drift per stream.
The results of the concept drift detectors are split up into two parts. For both parts, the detectors test every data dimension separately and not the performance values of classifiers.

The results of the first part of the evaluation are given in Tab. \ref{tab:confusion_matrix}. It shows a confusion matrix of the detectors tested on the listed concept drift stream generators given in Tab. \ref{tab:used-streams} plus the frequent reoccurring versions. The Kswin algorithm detect the concept drifts and has, by far, the most true positives but also the most false positive. The detection accuracy of Kswin is roughly $5\%$, while the detection rate of remaining detectors is roughly $0.001\%$. Summarizing, concept drift detection per dimension on stream data is somewhat unreliable and most of the time, the detectors just missing concept drift, because of insensitivity. The Kswin detector is an improvement to the state of the art detectors by means of detection rate. The false positive rate should be tackled in future work, but as shown in the prototype adaptation strategy in Sec. \ref{sec:properties_adaption} and shown in the following experiments, where Kswin is paired with the Naive Bayes classifier, false positives are not critical.
\begin{table}[b!]
    \centering
    \begin{tabular}{lllll}  \hline
        Detectors &True Negative&False Positive&False Negative&True Positive\\  \hline
    Kswin  &748956&51040&379345&20659 \\
    Adwin  &799823&173&399933&71 \\
    EDDM   &799941&55&399971&33 \\
    DDM     & 799984&12&400004&0 \\
      \hline
    \end{tabular}
   \caption{Confusion matrix of concept drift detectors showing that Kswin detects the most drifts but also has most false positives. It is tested on standard concept streams (10) and frequent reoccurring versions (6) of them. Overall, there are 400,004 concept drifts within 120,000,000 time steps. Each of the eight standard streams has one drift, and each of the frequent reoccurring concept drift streams has 99,999 drifts.}
    \label{tab:confusion_matrix}
\end{table}

The second part is shown in Tab. \ref{tab:detector_performance} and gives the prediction performance of the Naive Bayes classifier combined with a given detector. If a detector notices concept drift, the current underlying learning model of a respected detector is discarded and a new one is learned with the current batch of data. This experiment ensures that every detector uses the same underlying classifier and the classifier does not influence the performance.
\begin{table}[t!]
    \centering
    \begin{tabular}{lllll}  \hline
Standard Streams &Kswin&Adwin&EDDM&DDM\\  \hline
MIXED$_a$&0.8646&\textbf{0.9024}&0.5057&0.4998\\
MIXED$_g$&\textbf{0.8645}&0.5004&0.5000&0.5001\\
Hyperplane&0.5797&\textbf{0.6090}&0.5890&0.5883\\
RTG&0.6325&\textbf{0.7048}&0.6318&0.6655\\
RBF$_f$&\textbf{0.6584}&0.5299&0.4952&0.5009\\
RBF$_m$&0.6459&\textbf{0.6512}&0.5034&0.6019\\
SEA$_a$&0.8400&\textbf{0.8909}&0.7036&0.8876\\
SEA$_g$&0.8334&0.8304&0.6787&\textbf{0.8396}\\  \hline
Standard Mean &\textbf{0.7456}  &0.7089  &0.5853  &0.6435\\  \hline
MIXED$_{ra}$&\textbf{0.783}&0.4902&0.490&0.4896\\
MIXED$_{rg}$&\textbf{0.7766}&0.4921&0.4901&0.4975\\
SEA$_{ra}$&0.8385&\textbf{0.8766}&0.8104&0.8745\\
SEA$_{rg}$&0.8302&\textbf{0.8688}&0.806&0.847\\  \hline
Reoccurring Mean&\textbf{0.8217}  &0.7150 &0.6793 &0.7065\\ \hline
Overall Mean &\textbf{0.7837} &0.7120 &0.6323 &0.6750\\  \hline
    \end{tabular}
    \caption{Interleaved mean accuracy of concept drift detectors with Naive Bayes as the underlying classifier. Tested on standard concept drift streams and the frequent reoccurring versions of them separated in two parts. The results are overall similar comparing prediction performance, but in the mean, the Kswin detector outperforms the remaining detectors. Winner marked bold. a = abrupt, if = incremental fast, im = incremental medium, g = gradual, ra = frequent reoccurring}
    \label{tab:detector_performance}
\end{table}{}
Inspecting each drift separately, the Kswin algorithms is only best at four out of 16 streams. However, Kswin has the best mean performance. The second best algorithm is Adwin. The results at MIXED are interesting because the algorithms that do not recognize the concept drift, are not able to switch to the new concept. EDDM and DDM are particularly affected by this and also Adwin at frequent reoccurring MIXED. However, apart from MIXED, the performance of the detectors on the frequently reoccurring drift streams is not much worse compared to the standard streams.
\subsection{Comparison of Stream Classifiers}\label{sec:comparison_cls}
In these experiments, we compared the RRSLVQ with Hoeffding Adaptive Tree \cite{Bifet2009}, OzaBaggingAdwin \cite{DBLP:conf/smc/Oza05}, Adaptive Random Forest \cite{Gomes2017}, SamKNN \cite{Losing2017} and RSLVQ \cite{Seo2003} as the baseline.

In this setting, the frequent reoccurring drifts starting at sample 2000 and, further, every 1000 samples after the last drift is finished. The width of gradual drift is 1000. For every non-frequent reoccurring stream, we are adding abrupt/gradual concept drift at position 500,000, while gradual drift has a width of 50,000.

The RRSLVQ provides stable performance during drift and high adaptation rate shown in Fig. \ref{Fig:PerformanceRC}, superior w.r.t. other methods. Other variants of concept drift handling, like prototype insertion or prototype replacement based on the Adwin detector are not providing stability during a drift. Hence, both parts, the Kswin and the prototype adaptation strategy, are equally relevant for the stability of RRSLVQ.
\begin{figure}[b!]
  \centering
  \includegraphics[width=1\textwidth]{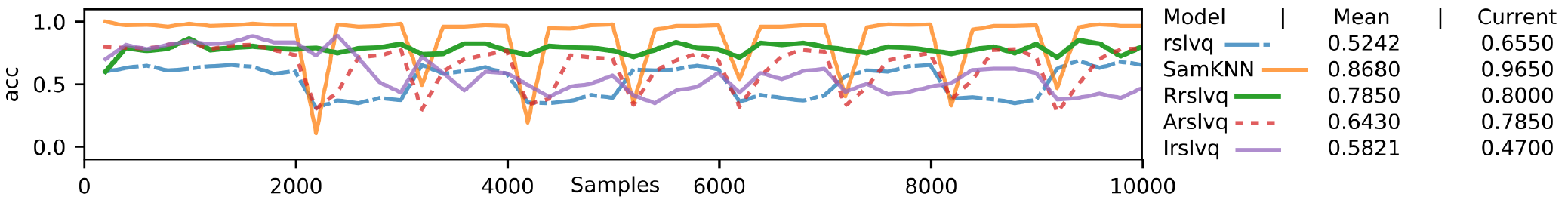}
  \caption{Performance of baseline RSLVQ, RRSLVQ, Incremental-RSLVQ, Adwin-RSLVQ, and SamKNN \cite{Losing2017a} on Mixed stream in SciKit-Multiflow \cite{skmultiflow}. Plot shows clear drops in performance of non-RRSLVQ methods during abrupt concept drift. From point 2000, drift happens every 1000 points after last drift is finished. The line shows accuracy over the last 200 samples. Best viewed in color.}\label{Fig:PerformanceRC}
\end{figure}

An overview of memory consumption during stream processing is shown in Fig. \ref{fig:memory_usage}. It shows that the RSLVQ variants and HAT have very low memory requirements and are stable in memory consumption during the drift. This validates Sec. \ref{sec:memory_time_complexity}. The stable consumption is not given by the remaining classifiers and at the detection of drifts, the memory usage fluctuates.

 \begin{figure}[b!]
    \centering     
    \begin{subfigure}[]{.47\linewidth}
    \includegraphics[width=1\textwidth]{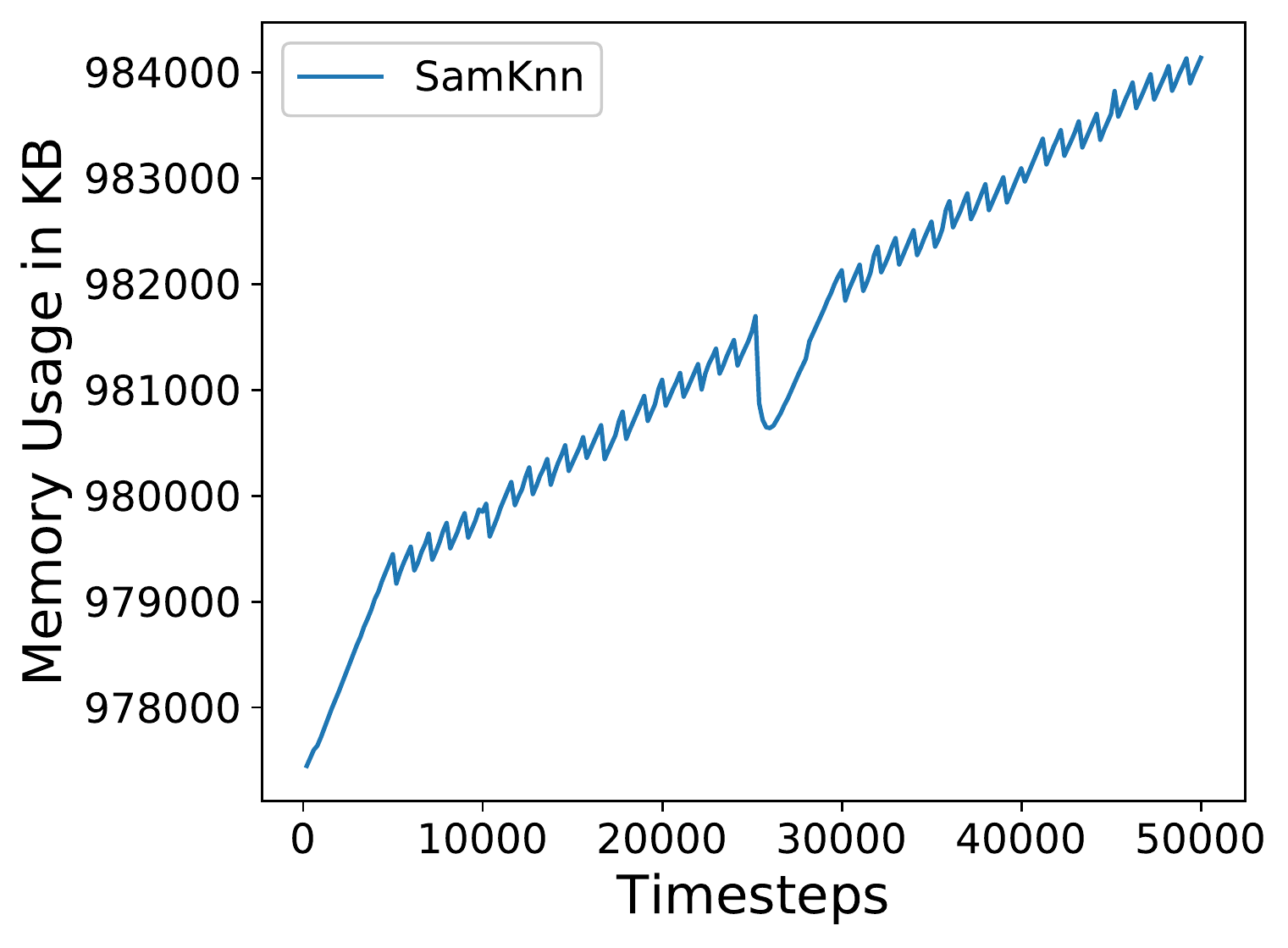}
    \subcaption{Memory SamKnn Mixed$_a$ \label{fig:samknn_a}}
    \end{subfigure}
    \begin{subfigure}[]{.47\linewidth}
    \includegraphics[width=1\textwidth]{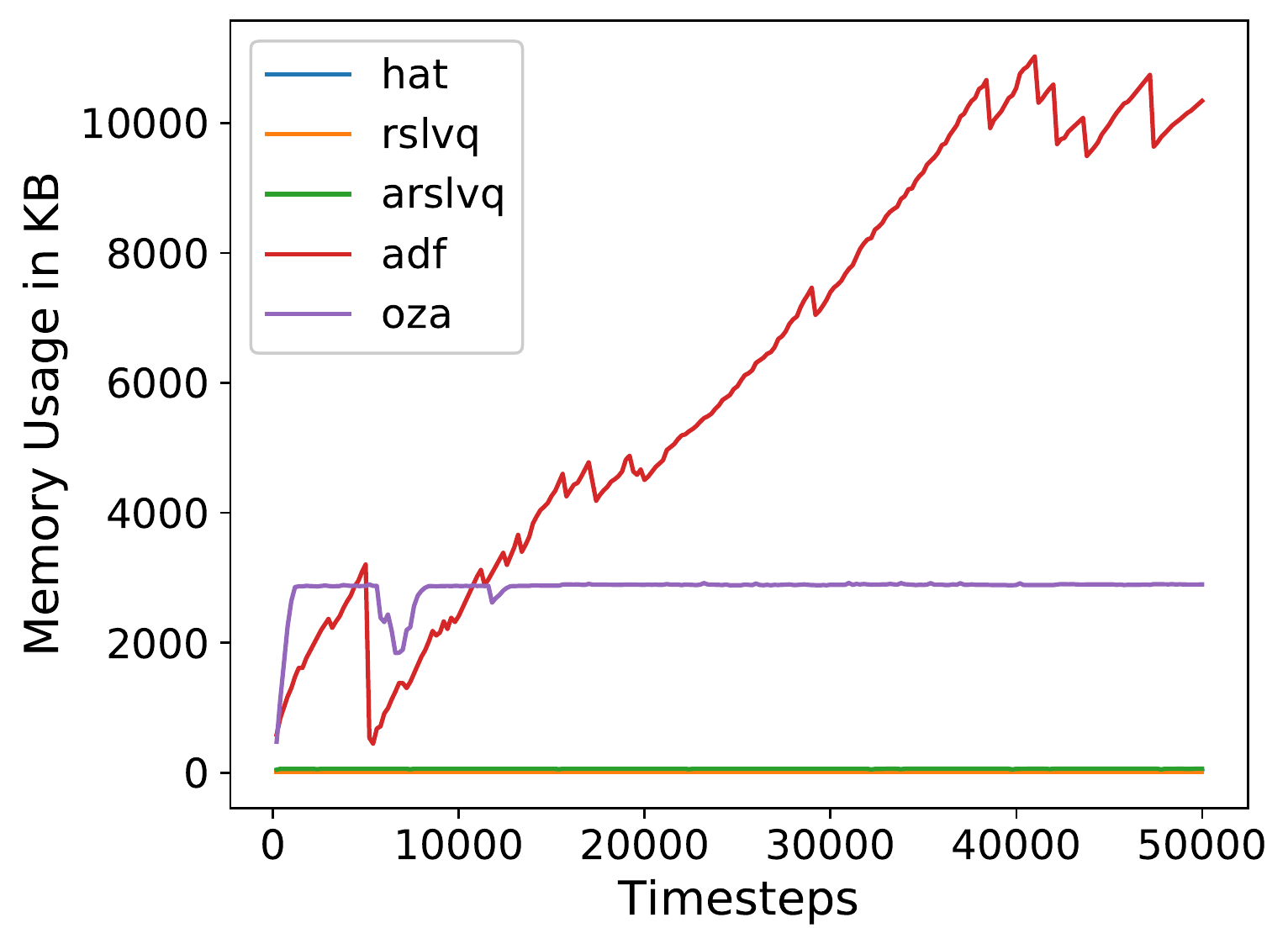}
    \subcaption{Memory Others Mixed$_a$ \label{fig:others_a}}
    \end{subfigure}
    \begin{subfigure}[]{.47\linewidth}
    \includegraphics[width=1\textwidth]{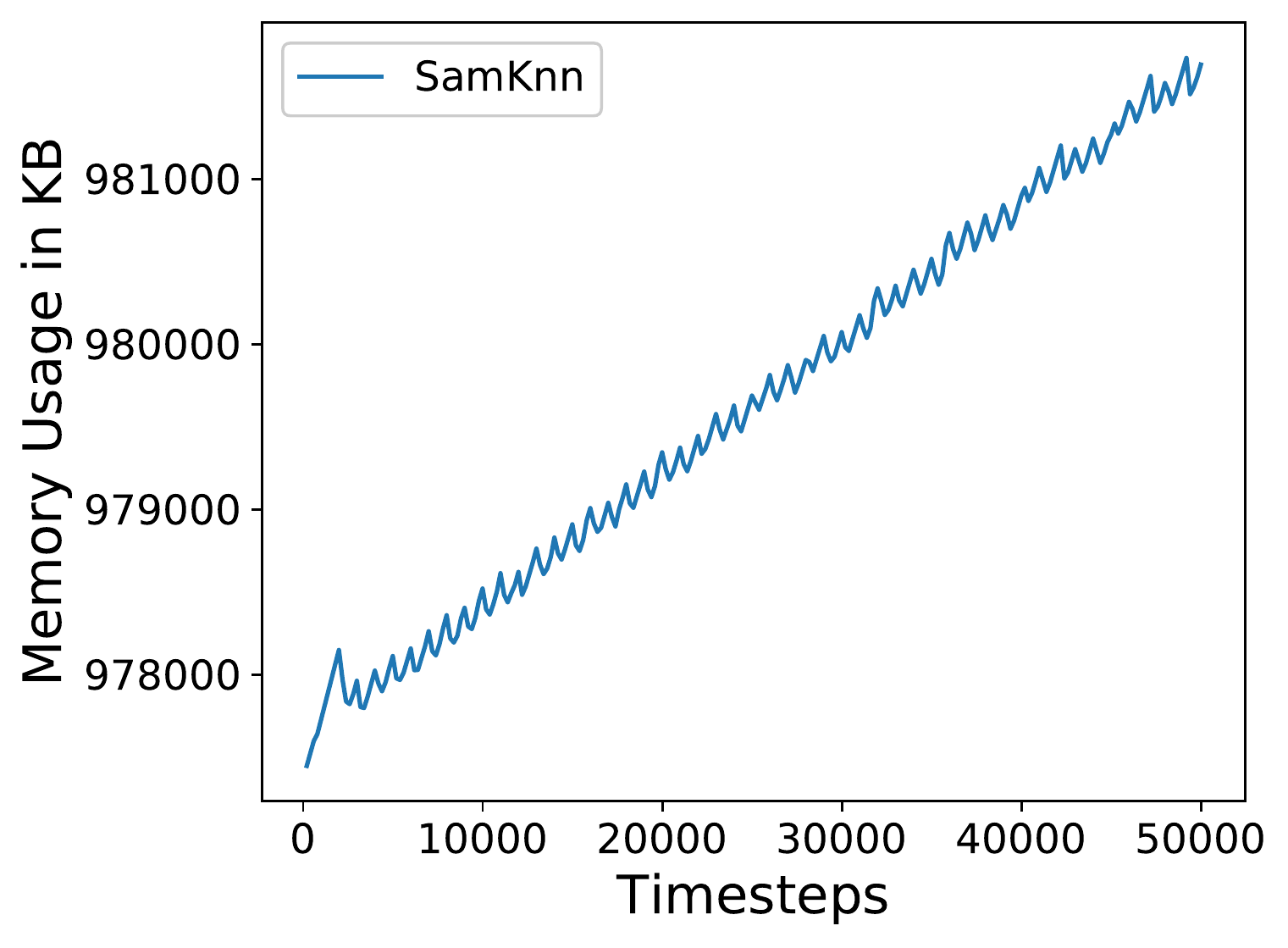}
    \subcaption{Memory SamKnn Mixed$_{ra}$ \label{fig:samknn_ra}}
    \end{subfigure}
    \begin{subfigure}[]{.47\linewidth}
    \includegraphics[width=1\textwidth]{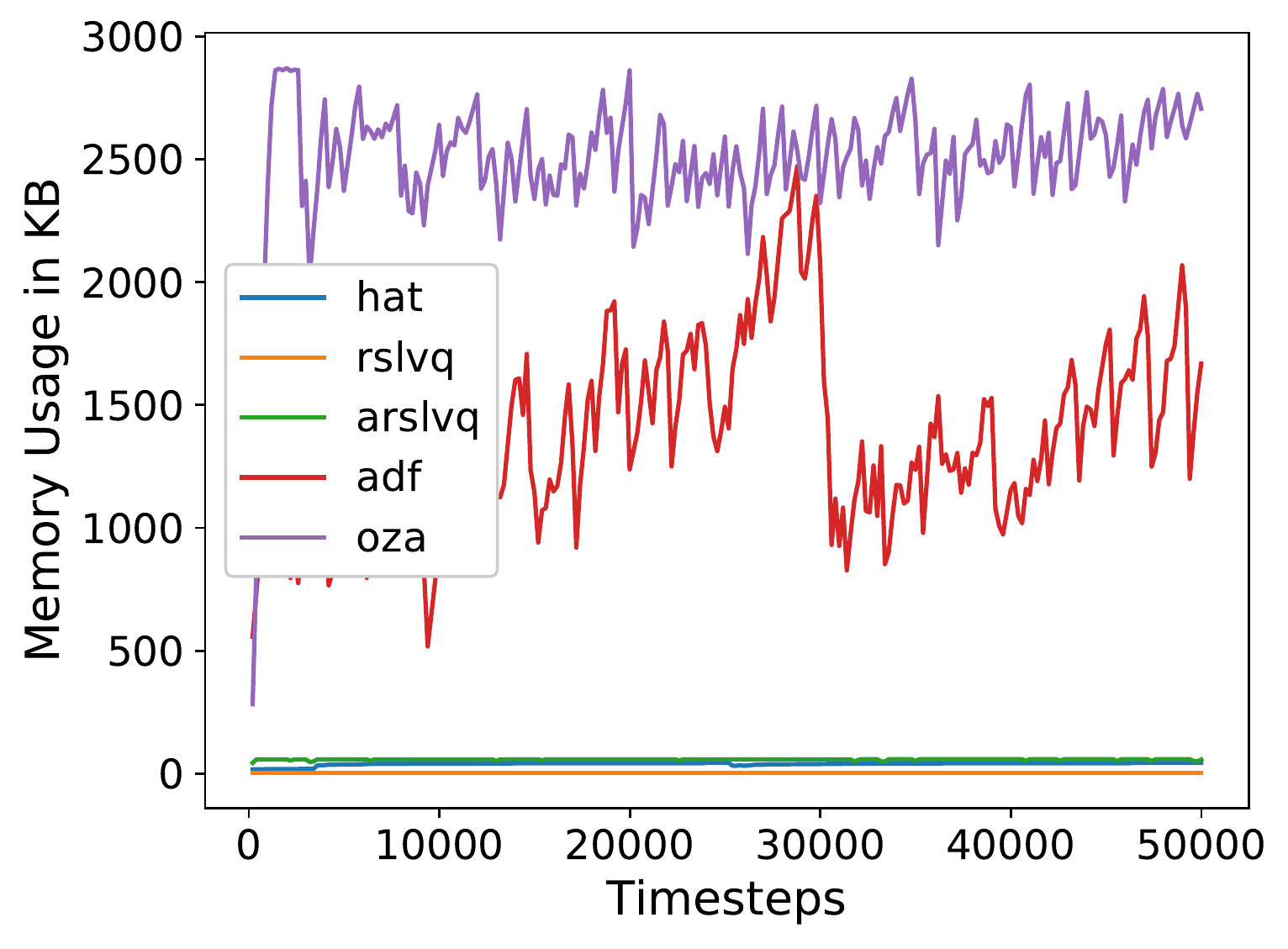}
    \subcaption{Memory Others Mixed$_{ra}$\label{fig:others_ra}}
    \end{subfigure}
    \caption{Memory usage of stream classifiers tested on Mixed stream with abrupt drift (top) and frequent abrupt reoccurring drift (bottom). SAMKNN is plotted separately due to scaling issues. It shows the constant memory consumption of RSLVQ methods and HAT even during drift. \label{fig:memory_usage}}
\end{figure}

The time result is plotted in Fig. \ref{fig:time} as incremental time in seconds per processed samples. It also validates the time complexity of RRSLVQ and further shows that every other stream classifier also has linear time complexity. However, the RRSLVQ needs less total time as the ARF and OZA. The RSLVQ is slightly faster as the RRSLVQ because of missing concept drift detector, non-momentum-based SGD and prototype adaptation strategy. The HAT algorithm is the fastest. All in all, the classifiers do not need any additional time to handle the concept drift.
\begin{figure}[]
    \centering     
    \begin{subfigure}[]{.45\linewidth}
    \includegraphics[width=1\textwidth]{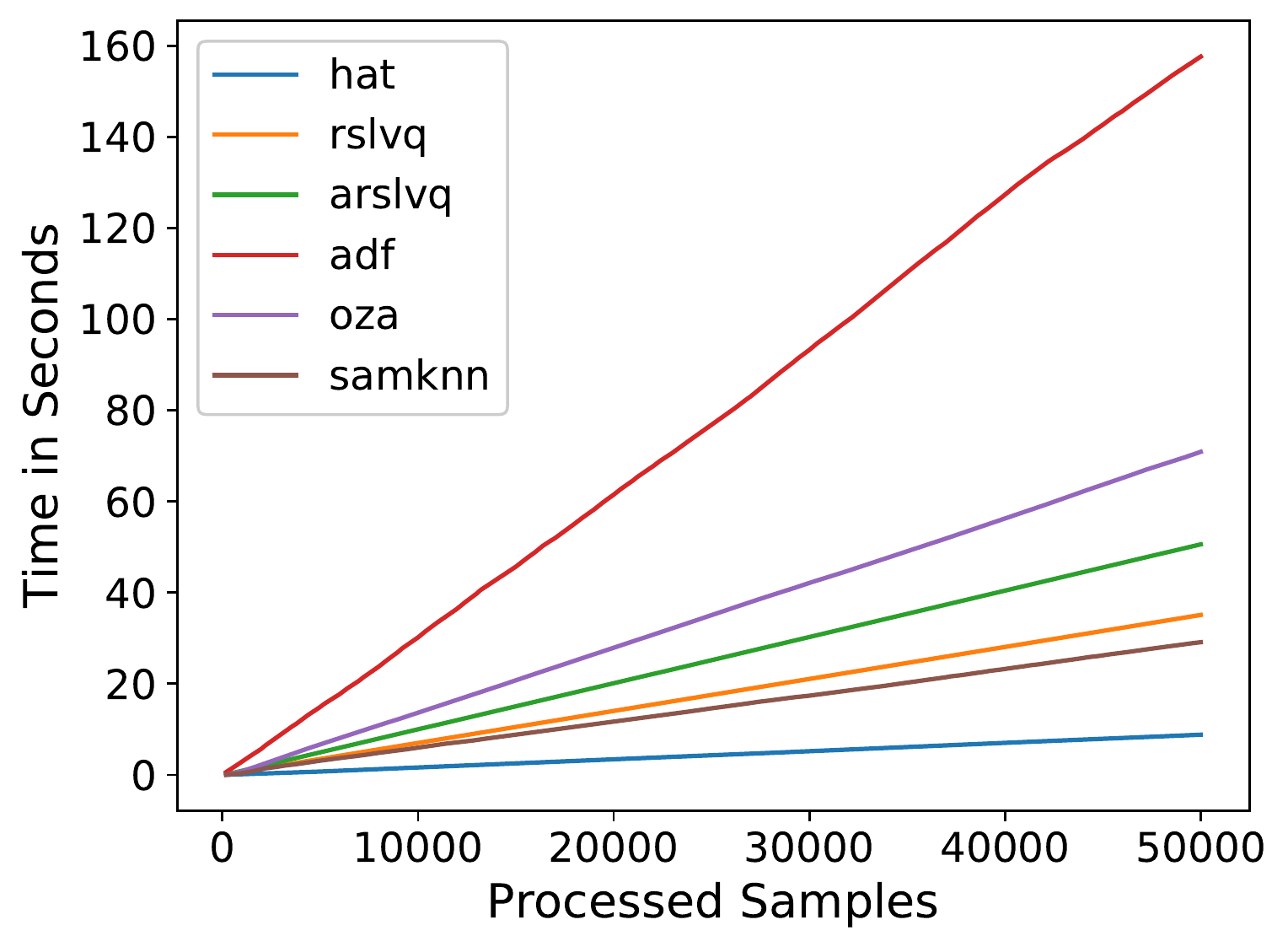}
    \subcaption{Incremental Time Mixed$_a$ \label{fig:abrupt_time}}
    \end{subfigure}
    \begin{subfigure}[]{.45\linewidth}
    \includegraphics[width=1\textwidth]{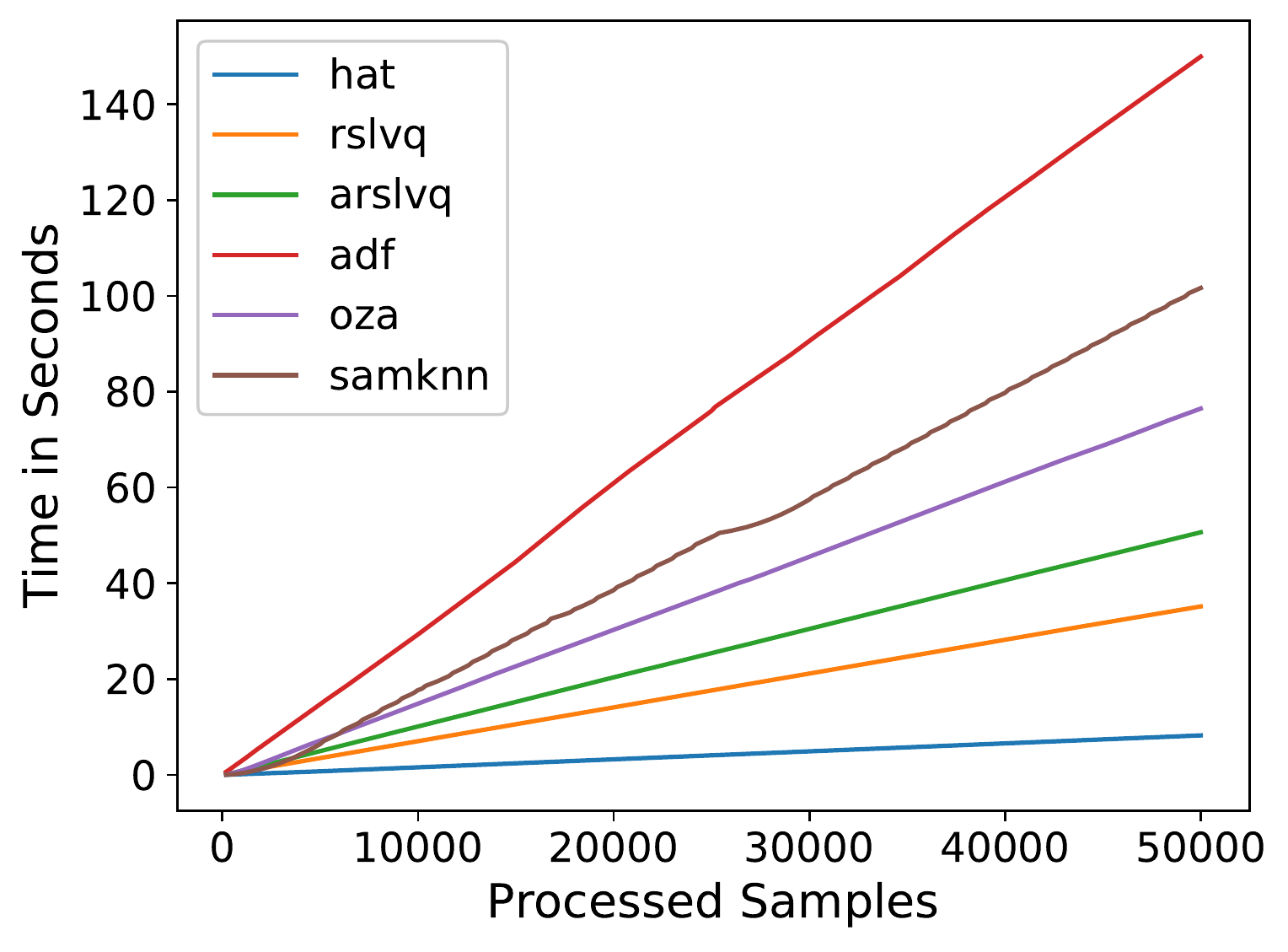}
    \subcaption{Incremental Time Mixed$_{ra}$ \label{fig:ra_time}}
    \end{subfigure}
   \caption{Incremental time in seconds per amount of samples processed on Mixed generator with abrupt drift (left) and frequent reoccurring drift (right). All tested methods are increasing linearly in time due to linear complexity per batch, even during drift. \label{fig:time}}
\end{figure}

The results of the prediction performance of the concept drift classifiers on the data streams are presented in Tab. \ref{tab:results-on-streams-acc}. Our approach shows a boost in performance to baseline RSLVQ and is comparable to other concept drift classifiers like HAT and OZA.
\begin{table}[b!]
  \resizebox{\textwidth}{!}{%
  \begin{tabular}{@{\extracolsep{\fill}}lllllll@{}}
  \hline
  Streams & ARF \cite{Gomes2017} & SAMKNN \cite{Losing2017a} & HAT \cite{Bifet2009} & RSLVQ \cite{Seo2003} & \underline{RRSLVQ} & OZA \cite{DBLP:conf/smc/Oza05}\\
  \hline
COV & 0.699 & \textbf{0.8937} & 0.8221 & 0.3647 & 0.7278  &0.6909
\\
ELEC & \textbf{0.856} & 0.7250 & 0.7750 & 0.6220 & 0.6448 & 0.740\\
POKER & \textbf{0.8071} & 0.7951 & 0.6455 & 0.7206 & 0.5779 & 0.7904\\
WTHR & 0.7392 & 0.7784 & 0.6771 & 0.6502 & 0.6703& \textbf{0.7827}\\
GMSC & \textbf{0.930} & 0.9270 & 0.8800 & 0.7800 & 0.9213  & 0.9196
\\
SQR & 0.5114 & \textbf{0.9659} & 0.7296 & 0.334 & 0.5993 & 0.4439\\
  \hline
REAL MEAN & 0.7572 & \textbf{0.8476} & 0.7548 & 0.5787 & 0.6838 & 0.7556\\
  \hline
MIXED$_a$& 0.934& \textbf{0.990}& 0.870& 0.733& 0.8930 & 0.975\\
MIXED$_g$& 0.912& \textbf{0.990}& 0.870& 0.733& 0.8938  & 0.975\\
Hyperplane & 0.5091 & 0.5602 & 0.622 & 0.5254 & \textbf{0.6208} & 0.5731\\
RTG & 0.5827 & {0.7008} & 0.6582 & 0.5769 & \textbf{0.8075} & 0.656\\
RBF$_{if}$ & 0.7457 & \textbf{0.9371} & 0.6242 & 0.5939 & 0.6648 & 0.9353\\
RBF$_{im}$ & 0.7983 & 0.9389 & 0.7344 & 0.5992 & 0.6797& \textbf{0.9557}\\
SEA$_a$ & 0.8866 & 0.8898 & 0.8362 & 0.8077 & \textbf{0.8924} & 0.881\\
SEA$_g$ & 0.8764 & 0.8869 & 0.8295 & 0.8033 & \textbf{0.8996} & 0.8747\\
MIXED$_{ra}$ & 0.8492 & \textbf{0.8535} & 0.5288 & 0.7316 & 0.7794  & 0.5464\\
MIXED$_{rg}$ & \textbf{0.8816} & 0.8535 & 0.5288 & 0.7227 & 0.7790 & 0.5468\\
SEA$_{ra}$ & 0.8668 & 0.8784 & 0.8312 & 0.8019 & 0.8857 & \textbf{0.9471}\\
SEA$_{rg}$ & 0.851 & 0.8808 & 0.8273 & 0.8049 & 0.8908 & \textbf{0.9471}\\
  \hline
ARTIFICIAL MEAN & 0.8078 & \textbf{0.8633} & 0.73 & 0.6899 & 0.8072 & 0.8178\\
  \hline
OVERALL MEAN & 0.7909 & \textbf{0.8581} & 0.7383 & 0.6528 &  0.7455 & 0.7971\\
  \hline
    \end{tabular}}
  \caption{Interleaved mean accuracy on data streams. Moving \underline{average} of accuracy on one million samples. Winner marked bold. a = abrupt, if = incremental fast, im = incremental medium, g = gradual, ra = frequent reoccurring}
  \label{tab:results-on-streams-acc}
  \end{table}
The RSLVQ is the worst classifier on the tested real-world datasets, affected by the worse performance on COV and SQR. The RRSLVQ performs about 8 \% better when looking at the mean of the real-world datasets. However, it is still worse than the other classifiers, especially SAMKNN, which performed best. Especially on COV, where RSLVQ had an accuracy of 36 \%, RRSLVQ improved the prediction to an accuracy of 73 \%, which is approximately equal to the performance of ARF. On the artificial data streams, the RSLVQ performs worst with a mean accuracy of 69 \%. The RRSLVQ does a better job on the synthetic concept drift streams due to the drift detector and thus has better accuracy than RSLVQ and HAT with 81 \%. Again,  SAMKNN performed best with an accuracy of 86 \%. Overall, RRSLVQ performed approximately equal to OZA with an accuracy of 81 \%, which is a performance increase of 12 \% compared to the baseline RSLVQ.

Comparing the prediction results in Tab. \ref{tab:results-on-streams-acc} to the time and memory consumption in Fig. \ref{fig:time} and \ref{fig:memory_usage}, it encourages the assumption that there is a trade-off between prediction performance and model size and time.

The results of the stream experiment measured by Kappa statistics are shown in Tab. \ref{tab:results-on-streams-kappa}. The improved performance of RRSLVQ compared to RSLVQ is also given w.r.t. Kappa. This means that RRSLVQ also performs better on imbalanced data and has no tendency to only predict one label. Especially on COVTYPE, where RSLVQ was not able to separate the classes of the dataset, RRSLVQ showed a Kappa score of 45.98 \%.  The improved score is not only related to frequent reoccurring drift. It is also given when using other types of drift. On the Hyperplane generator, RRSLVQ performed best and is the only algorithm with a Kappa score $> 0.2$. This is due to the case that the distribution of the stream changes naturally very often and thus needs a sensible concept drift detector. However, the Kappa score is still very low on this dataset. Given the performance of the baseline classifier, the RRSLVQ does a good job of improving the RSLVQ on concept drift streams. On the real-world streams, the Kappa score is also better, but the overall improvement is small. This is affected by the fact that RSLVQ did a better job of distinguishing the ten classes of the POKER dataset. Overall, SAMKNN performed best with a score of 65 \%, while RRSLVQ performed very similar to HAT at 45 \%

\begin{table}[t!]
  \resizebox{\textwidth}{!}{%
  \begin{tabular}{@{}l llllll@{}}
  \hline
  Streams & ARF \cite{Gomes2017} & SAMKNN \cite{Losing2017a} & HAT \cite{Bifet2009} & RSLVQ \cite{Seo2003} & \underline{RRSLVQ} & OZA \cite{DBLP:conf/smc/Oza05}\\
    \hline
COV & 0.4267 & \textbf{0.8285} & 0.7139 & 0 & 0.4598 & 0.8233\\
ELEC & \textbf{0.6981} & 0.4334 & 0.5347 & 0.2158 & 0.2616 & 0.4609\\
POKER & \textbf{0.6673} & 0.6259 & 0.3689 & 0.4969 & 0.3498 & 0.6162\\
WTHR & 0.3352 & 0.4439 & 0.2634 & 0.268 & 0.2743 & \textbf{0.47}\\
GMSC & \textbf{0.0788} & 0.002 & 0.0069 & -0.0019 & 0.0017 & 0.0026\\
SQRE & 0.3486 & \textbf{0.9546} & 0.6394 & 0.112 & 0.209 & 0.2054\\
  \hline
REAL MEAN & 0.4258 & \textbf{0.5481} & 0.4212 & 0.1818 & 0.2594 & 0.4297\\
  \hline
MIXED$_{a}$ & 0.8685 & \textbf{0.9791} & 0.7392 & 0.4643 & 0.8117 & 0.9505\\
MIXED$_{g}$ & 0.823 & \textbf{0.9791} & 0.7392 & 0.4643 & 0.8138 & 0.9505\\
Hyperplane & 0.0175 & 0.1204 & 0.2441 & 0.0508 & \textbf{0.2555} & 0.1462\\
RTG & 0.1631 & \textbf{0.3441} & 0.273 & 0.1315 & 0.2574 & 0.1723\\
RBF$_{if}$ & 0.4701 & \textbf{0.8736} & 0.2434 & 0.1861 & 0.2457 & 0.8658\\
RBF$_{im}$ & 0.594 & 0.8773 & 0.4297 & 0.1601 & 0.2831 & \textbf{0.9063}\\
SEA$_{a}$ & \textbf{0.7454} & 0.7282 & 0.6029 & 0.5434 & 0.7398 & 0.7082\\
SEA$_{g}$ & 0.723 & 0.7626 & 0.6435 & 0.5905 & \textbf{0.7896} & 0.7376\\
MIXED$_{ra}$ & 0.6982 & \textbf{0.707} & 0.0576 & 0.2998 & 0.4445 & 0.0928\\
MIXED$_{rg}$ & \textbf{0.7631} & 0.707 & 0.0576 & 0.2989 & 0.4313 & 0.0936\\
SEA$_{ra}$ & 0.6997 & 0.7143 & 0.609 & 0.5780 & 0.7324 & \textbf{0.8851}\\
SEA$_{rg}$ & 0.6561 & 0.7426 & 0.6286 & 0.5802 & 0.7646 &\textbf{ 0.8851}\\
  \hline
ARTIFICIAL MEAN & 0.6018 & \textbf{0.7113} & 0.439 & 0.3623 & 0.5474 & 0.6162\\
  \hline
OVERALL MEAN & 0.5431 & \textbf{0.6569} & 0.4331 & 0.3022 & 0.4515 & 0.554\\
  \hline

    \end{tabular}}
  \caption{Interleaved test-then-train Kappa statistics on data streams. Moving \underline{average} of Kappa on one million samples. Winner marked bold. a = abrupt, if = incremental fast, im = incremental medium, g = gradual, ra = frequent reoccurring}
  \label{tab:results-on-streams-kappa}
\end{table}

Summarizing, the advantages of the RRSLVQ are the stability during concept drift and the constant model size and time. This makes the processing of stream data on a limited technical device feasible. The prototypes are providing an interpretable model, which is a lacking feature of the competitive algorithms. The prediction performance is comparable with OZA and HAT.

\section{Conclusion}
The proposed method is a major improvement to the original RSLVQ. Especially in streams with high rates of drift, the RRSLVQ\footnote{Source code available at https://github.com/ChristophRaab/rrslvq} shows remarkable stability over time.

Kswin seems to detect occurring changes in stream data and supports the concept drift handling process with good indicators at a given time.
Based on the experimental setting, the Kswin algorithm is the best at detecting drift at the given streaming data.
Further, the adaptation strategy paired with the momentum-based gradient descent is useful to minimize performance losses during the drift.

Compared to other stream approaches, the RRSLVQ provides a straightforward and interpretable model. The memory and time complexity is easy to bound and well-suited for embedding systems. The prediction performance is comparable with OzaBaggingAdwin and Hoeffding Adaptive Tree.

Future work should tackle dimension-wise testing at every time step to avoid unneeded tests. Besides, Kswin should be combined with other classifiers and false positives should be inspected. Although that Kswin already outperforms standard detectors, the task of concept drift detection must still be improved.

\section*{Acknowledgement}
We are thankful for support in the FuE program Informations- und Kommunikationstechnik of the StMWi, project OBerA, grant number IUK-1709-0011// IUK530/010.

\appendix
\section{Dataset Description}\label{sec:dataset_description}
In this appendix, the data streams are described. It is split up in synthetic data streams and real-world streams.
\subsection{Synthetic Stream Generators}
The synthetic stream generators are potentially infinite and drift can be implemented by changing the generating function. A common approach to create drift is to invert the function. The length of the streams is set to one millions time steps.
\textbf{SEA}
The SEA generator is an implementation of the data stream with abrupt concept drift, first described by Street and Kim in \cite{Street2001SEAGenerator}. It produces data streams with three continuous attributes $(f_1, f_2, f_3 )$. The range of values that each attribute can assume lies between 0 and 10. Only the first two attributes $(f_1, f_2)$ are relevant, i.e.\, $f_3$ does not influence the class value determination. New instances are obtained through randomly setting a point in a two-dimensional space, such that these dimensions correspond to $f_1$ and $f_2 $. This two-dimensional space is split into four blocks, each of which corresponds to one of four different functions. In each block, a point belongs to class 1 if $f_1 + f_2 \leq \theta$ and to class 0 otherwise. The threshold $\theta$ is used to split instances between class 0 and 1 assumes the values 8 (block 1), 9 (block 2), 7 (block 3), and 9.5 (block 4). Two important features are the possibility to balance classes, which means the class distribution will tend to a uniform one, and the possibility to add noise, which will, according to some probability, change the chosen label for an instance. In this experiment, the SEA generator is used with 10 \% noise in the data stream. $\mathrm{SEA}_g$ simulates one gradual drift, while $\mathrm{SEA}_a$ simulates an abrupt drift.


\textbf{MIXED} The MIXED Generator creates a binary classification stream. The stream consists of four features, two Boolean attributes $v$, $w$ and two numeric attributes $x$ and $y$ between $[0; 1]$. The label is positive if two out of three conditions are satisfied
\begin{equation}
    v = true , t=true , y <  0.5 + 0.3 \sin{(3 \pi x)}.
\end{equation}
After each concept drift, the classification is reversed \cite{Gama2004}.

\textbf{RTG}
The Random Tree Generator (RTG) is based on a random tree that splits features at random and sets labels to its leafs \cite{Domingos.2000}. After the tree is built, new instances are obtained through the assignment of uniformly distributed random values to each attribute. The leaf reached after a traverse of the tree, according to the attribute values of an instance, determines its class value. RTG allows customizing the number of nominal and numeric attributes, as well as the number of classes. In our experiments, we did not simulate drifts for the RTG data set. Since the concepts are generated and classified according to a tree structure, in theory, it should favor decision tree learners.

\textbf{RBF}
This generator produces data sets by means of the Radial Basis Function (RBF) \cite{skmultiflow}. This generator creates several centroids, having a random central position and associates them with a standard deviation value, a weight, and a class label. To create new instances, one centroid is selected at random, where centroids with higher weights have more chances to be selected. The new instance input values are set according to a random direction chosen to offset the centroid. The extent of the displacement is randomly drawn from a Gaussian distribution according to the standard deviation associated with the given centroid. Incremental drift is introduced by moving centroids at a continuous rate, effectively causing new instances that ought to belong to one centroid to another with (maybe) a different class. Both $\mathrm{RBF}_m$ and $\mathrm{RBF_f}$ were parametrized with 50 centroids, and all of them drift. $\mathrm{RBF}_m$ simulates a moderate incremental drift (speed of change set to 0.0001) while $\mathrm{RBF}_f$ simulates a faster incremental drift (speed of change set to 0.001).

\textbf{HYPER}
The HYPER data set simulates an incremental drift, and it was generated based on the hyperplane generator \cite{Hulten2001Hyperplane}. A hyperplane is a flat, $n - 1$ dimensional subset of that space that divides it into two disconnected parts. It is possible to change a hyperplane orientation and position by slightly changing its relative size of the weights $w_i$. This generator can be used to simulate time-changing concepts by varying the values of its weights as the stream progresses \cite{Bifet2018MOA}. HYPER was parametrized with ten attributes and magnitude of change of 0.001. Also, a 10 \% noise was added.

\subsection{Real-world Streams}
The real-world streams are finite in length and concept drift from a statistical standpoint is unknown. However, some streams have scenarios, which change by nature like COVTYPE.

\textbf{GMSC}
The Give Me Some Credit (GMSC) data set\footnote{https://www.kaggle.com/c/GiveMeSomeCredit} is a credit scoring data set where the objective is to decide whether a loan should be allowed or not. This decision is essential for banks since erroneous loans lead to the risk of default and unnecessary expenses on future lawsuits. The data set contains historical data on 150,000 borrowers, each described by ten attributes.

\textbf{Electricity} The Electricity data set\footnote{https://www.openml.org/d/151} was collected from the Australian New South Wales Electricity Market, where prices are not fixed. These prices are affected by the demand and supply of the market itself and set every five minutes. The Electricity data set contains 45,312 instances, where class labels identify the changes in the price (two possible classes: up or down) relative to a moving average of the last 24 hours. An important aspect of this data set is that it exhibits temporal dependencies \cite{Gama2004}.

\textbf{Poker-Hand} The Poker-Hand data set consists of 1,000,000 instances and eleven attributes. Each record of the Poker-Hand data set is an example of a hand consisting of five playing cards drawn from a standard deck of 52. Each card is described using two attributes (suit and rank), for a total of ten predictive attributes. There is one class attribute that describes the \textit{Poker Hand}.
This data set has no drift in its original form since the poker hand definitions do not change, and the instances are randomly generated. Thus, the version presented in \cite{Bifet.2013} is used, in which virtual drift is introduced via sorting the instances by rank and
suit. Duplicate hands were removed.

\textbf{Forest Cover Type} Forest Cover Type (COVTYPE) is a data set, which is often used to benchmark stream mining algorithms \cite{Bifet.2013, Oza2001}. The dataset assigns cartographic variables like elevation, soil type, slope, etc. of 30  square meter cells to different forest cover types. It only includes forests with minimal human-caused disturbances, so that the contained cover types are mostly a result of ecological processes.

\textbf{Airlines} The task on the Airlines dataset is to predict whether a planned flight will be delayed or not. It contains the airport of departure and arrival, as well as the airline and time-related features. The delay is only represented as a binary attribute. The dataset is used in the form as the MOA framework \cite{Bifet2018MOA} provides it.

\textbf{SQRE} The Moving Squares dataset has been used in \cite{Losing2017a} and consists of four equally distant separated squared uniform distributions which are moving in a horizontal direction with constant speed. Whenever the leading square reaches a predefined boundary, the direction is inverted. Each square represents a different class. The added value of this dataset is the predefined time horizon of 120 examples before old instances may start to overlap current ones. Thus, the dataset should be useful for dynamic sliding window approaches, allowing testing whether the size is adjusted accordingly.




\bibliographystyle{elsarticle-num-names}
\bibliography{library}

\begin{thebibliography}{35}
\providecommand{\natexlab}[1]{#1}
\providecommand{\url}[1]{\texttt{#1}}
\providecommand{\urlprefix}{URL }
\expandafter\ifx\csname urlstyle\endcsname\relax
  \providecommand{\doi}[1]{doi:\discretionary{}{}{}#1}\else
  \providecommand{\doi}[1]{doi:\discretionary{}{}{}\begingroup
  \urlstyle{rm}\url{#1}\endgroup}\fi
\providecommand{\bibinfo}[2]{#2}

\bibitem[{Raab and Schleif(2019)}]{esann18_raab}
\bibinfo{author}{C.~Raab}, \bibinfo{author}{F.-M. Schleif},
  \bibinfo{title}{{Reactive Soft Prototype Computing for frequent reoccurring
  Concept Drift}}, in: \bibinfo{booktitle}{27th European Symposium on
  Artificial Neural Networks, {ESANN} 2019, Bruges, Belgium, April 24-26,
  2019}, \bibinfo{pages}{437--442}, \bibinfo{year}{2019}.

\bibitem[{Gama et~al.(2014)Gama, Zliobaite, Bifet, Pechenizkiy, and
  Bouchachia}]{Gama2014}
\bibinfo{author}{J.~Gama}, \bibinfo{author}{I.~Zliobaite},
  \bibinfo{author}{A.~Bifet}, \bibinfo{author}{M.~Pechenizkiy},
  \bibinfo{author}{A.~Bouchachia}, \bibinfo{title}{{A survey on concept drift
  adaptation}}, \bibinfo{journal}{ACM Computing Surveys}
  \bibinfo{volume}{46}~(\bibinfo{number}{4}) (\bibinfo{year}{2014})
  \bibinfo{pages}{1--37}.

\bibitem[{Seo et~al.(2003)Seo, Bode, and Obermayer}]{Seo2003}
\bibinfo{author}{S.~Seo}, \bibinfo{author}{M.~Bode},
  \bibinfo{author}{K.~Obermayer}, \bibinfo{title}{{Soft nearest prototype
  classification}}, \bibinfo{journal}{IEEE Transactions on Neural Networks}
  \bibinfo{volume}{14}~(\bibinfo{number}{2}) (\bibinfo{year}{2003})
  \bibinfo{pages}{390--398}.

\bibitem[{Straat et~al.(2018)Straat, Abadi, G{\"{o}}pfert, Hammer, and
  Biehl}]{Straat2018}
\bibinfo{author}{M.~Straat}, \bibinfo{author}{F.~Abadi},
  \bibinfo{author}{C.~G{\"{o}}pfert}, \bibinfo{author}{B.~Hammer},
  \bibinfo{author}{M.~Biehl}, \bibinfo{title}{Statistical Mechanics of On-Line
  Learning Under Concept Drift}, \bibinfo{journal}{Entropy}
  \bibinfo{volume}{20}~(\bibinfo{number}{10}) (\bibinfo{year}{2018})
  \bibinfo{pages}{775}.

\bibitem[{Bifet and Gavald{\`{a}}(2006)}]{Bifet2006}
\bibinfo{author}{A.~Bifet}, \bibinfo{author}{R.~Gavald{\`{a}}},
  \bibinfo{title}{{Kalman Filters and Adaptive Windows for Learning in Data
  Streams}}, in: \bibinfo{editor}{L.~Todorovski},
  \bibinfo{editor}{N.~Lavra{\v{c}}}, \bibinfo{editor}{K.~P. Jantke} (Eds.),
  \bibinfo{booktitle}{Discovery Science}, \bibinfo{publisher}{Springer Berlin
  Heidelberg}, \bibinfo{address}{Berlin, Heidelberg}, \bibinfo{pages}{29--40},
  \bibinfo{year}{2006}.

\bibitem[{dos Reis et~al.(2016)dos Reis, Flach, Matwin, and
  Batista}]{DosReis2016}
\bibinfo{author}{D.~M. dos Reis}, \bibinfo{author}{P.~Flach},
  \bibinfo{author}{S.~Matwin}, \bibinfo{author}{G.~Batista},
  \bibinfo{title}{{Fast Unsupervised Online Drift Detection Using Incremental
  Kolmogorov-Smirnov Test}}, \bibinfo{journal}{Proceedings of the 22nd ACM
  SIGKDD International Conference on Knowledge Discovery and Data Mining - KDD
  '16} \bibinfo{volume}{22}~(\bibinfo{number}{1}) (\bibinfo{year}{2016})
  \bibinfo{pages}{1545--1554}.

\bibitem[{Salperwyck et~al.(2015)Salperwyck, Boull{\'{e}}, and
  Lemaire}]{Salperwyck2015}
\bibinfo{author}{C.~Salperwyck}, \bibinfo{author}{M.~Boull{\'{e}}},
  \bibinfo{author}{V.~Lemaire}, \bibinfo{title}{{Concept drift detection using
  supervised bivariate grids}}, \bibinfo{journal}{Proceedings of the
  International Joint Conference on Neural Networks}
  \bibinfo{volume}{2015-September}.

\bibitem[{Losing et~al.(2017{\natexlab{a}})Losing, Hammer, and
  Wersing}]{Losing2017a}
\bibinfo{author}{V.~Losing}, \bibinfo{author}{B.~Hammer},
  \bibinfo{author}{H.~Wersing}, \bibinfo{title}{{KNN classifier with self
  adjusting memory for heterogeneous concept drift}},
  \bibinfo{journal}{Proceedings - IEEE International Conference on Data Mining,
  ICDM} \bibinfo{volume}{1} (\bibinfo{year}{2017}{\natexlab{a}})
  \bibinfo{pages}{291--300}.

\bibitem[{Climer and Mendenhall(2016)}]{Climer2016a}
\bibinfo{author}{J.~Climer}, \bibinfo{author}{M.~J. Mendenhall},
  \bibinfo{title}{{Dynamic Prototype Addition in Generalized Learning Vector
  Quantization}}, \bibinfo{journal}{Advances in Self-Organizing Maps and
  Learning Vector Quantization} \bibinfo{volume}{428} (\bibinfo{year}{2016})
  \bibinfo{pages}{355--368}.

\bibitem[{Nova and Est{\'{e}}vez(2014)}]{Nova2014}
\bibinfo{author}{D.~Nova}, \bibinfo{author}{P.~A. Est{\'{e}}vez},
  \bibinfo{title}{{A review of learning vector quantization classifiers}},
  \bibinfo{journal}{Neural Computing and Applications}
  \bibinfo{volume}{25}~(\bibinfo{number}{3-4}) (\bibinfo{year}{2014})
  \bibinfo{pages}{511--524}.

\bibitem[{Sato and Yamada(1995)}]{Sato1995}
\bibinfo{author}{A.~Sato}, \bibinfo{author}{K.~Yamada},
  \bibinfo{title}{{Generalized Learning Vector Quantization}},
  \bibinfo{journal}{Advances in Neural Information Processing Systems, NIPS}
  ~(\bibinfo{number}{8}) (\bibinfo{year}{1995}) \bibinfo{pages}{423--429}.

\bibitem[{Villmann et~al.(2017)Villmann, Bohnsack, and Kaden}]{Villmann2017}
\bibinfo{author}{T.~Villmann}, \bibinfo{author}{A.~Bohnsack},
  \bibinfo{author}{M.~Kaden}, \bibinfo{title}{{Can learning vector quantization
  be an alternative to SVM and deep learning? - Recent trends and advanced
  variants of learning vector quantization for classification learning}},
  \bibinfo{journal}{Journal of Artificial Intelligence and Soft Computing
  Research} \bibinfo{volume}{7}~(\bibinfo{number}{1}) (\bibinfo{year}{2017})
  \bibinfo{pages}{65--81}.

\bibitem[{Zell et~al.(1993)Zell, Mamier, H{\"{u}}bner, Schmalzl, Sommer, and
  Vogt}]{Zell1993}
\bibinfo{author}{A.~Zell}, \bibinfo{author}{G.~Mamier},
  \bibinfo{author}{R.~H{\"{u}}bner}, \bibinfo{author}{N.~Schmalzl},
  \bibinfo{author}{T.~Sommer}, \bibinfo{author}{M.~Vogt},
  \bibinfo{title}{{SNNS: An Efficient Simulator for Neural Nets}}, in:
  \bibinfo{booktitle}{MASCOTS '93, Proceedings of the International Workshop on
  Modeling, Analysis, and Simulation On Computer and Telecommunication Systems,
  January 17-20, 1993, La Jolla, San Diego, CA, USA},
  \bibinfo{pages}{343--346}, \bibinfo{year}{1993}.

\bibitem[{Losing et~al.(2015)Losing, Hammer, and Wersing}]{Losing2015}
\bibinfo{author}{V.~Losing}, \bibinfo{author}{B.~Hammer},
  \bibinfo{author}{H.~Wersing}, \bibinfo{title}{{Interactive online learning
  for obstacle classification on a mobile robot}},
  \bibinfo{journal}{Proceedings of the International Joint Conference on Neural
  Networks} \bibinfo{volume}{2015-September}~(\bibinfo{number}{2}).

\bibitem[{Bifet et~al.(2015)Bifet, Read, and Holmes}]{Bifet2015}
\bibinfo{author}{A.~Bifet}, \bibinfo{author}{J.~Read},
  \bibinfo{author}{G.~Holmes}, \bibinfo{title}{{Efficient Online Evaluation of
  Big Data Stream Classifiers Categories and Subject Descriptors}},
  \bibinfo{journal}{Proceedings of the 21th ACM SIGKDD International Conference
  on Knowledge Discovery and Data Mining}  (\bibinfo{year}{2015})
  \bibinfo{pages}{59--68}.

\bibitem[{Gomes et~al.(2017)Gomes, Bifet, Read, Barddal, Enembreck, Pfharinger,
  Holmes, and Abdessalem}]{Gomes2017}
\bibinfo{author}{H.~M. Gomes}, \bibinfo{author}{A.~Bifet},
  \bibinfo{author}{J.~Read}, \bibinfo{author}{J.~P. Barddal},
  \bibinfo{author}{F.~Enembreck}, \bibinfo{author}{B.~Pfharinger},
  \bibinfo{author}{G.~Holmes}, \bibinfo{author}{T.~Abdessalem},
  \bibinfo{title}{{Adaptive random forests for evolving data stream
  classification}}, \bibinfo{journal}{Machine Learning}
  \bibinfo{volume}{106}~(\bibinfo{number}{9}) (\bibinfo{year}{2017})
  \bibinfo{pages}{1469--1495}.

\bibitem[{Gama et~al.(2004{\natexlab{a}})Gama, Medas, Castillo, and
  Rodrigues}]{Gama.04}
\bibinfo{author}{J.~Gama}, \bibinfo{author}{P.~Medas},
  \bibinfo{author}{G.~Castillo}, \bibinfo{author}{P.~Rodrigues},
  \bibinfo{title}{{Learning with Drift Detection}}, in:
  \bibinfo{editor}{A.~L.~C. Bazzan}, \bibinfo{editor}{S.~Labidi} (Eds.),
  \bibinfo{booktitle}{Advances in Artificial Intelligence -- SBIA 2004},
  \bibinfo{publisher}{Springer Berlin Heidelberg}, \bibinfo{address}{Berlin,
  Heidelberg}, \bibinfo{pages}{286--295}, \bibinfo{year}{2004}{\natexlab{a}}.

\bibitem[{Wankhade et~al.(2013)Wankhade, Hasan, and Thool}]{Wankhade.2013}
\bibinfo{author}{K.~Wankhade}, \bibinfo{author}{T.~Hasan},
  \bibinfo{author}{R.~Thool}, \bibinfo{title}{{A Survey: Approaches for
  Handling Evolving Data Streams}}, in: \bibinfo{booktitle}{2013 International
  Conference on Communication Systems and Network Technologies},
  \bibinfo{pages}{621--625}, \bibinfo{year}{2013}.

\bibitem[{Lopes(2011)}]{Lopes2011}
\bibinfo{author}{R.~H.~C. Lopes}, \bibinfo{title}{Kolmogorov-Smirnov Test},
  \bibinfo{publisher}{Springer Berlin Heidelberg}, \bibinfo{address}{Berlin,
  Heidelberg}, ISBN \bibinfo{isbn}{978-3-642-04898-2},
  \bibinfo{pages}{718--720}, \bibinfo{year}{2011}.

\bibitem[{Abdi(2007)}]{Abdi2007}
\bibinfo{author}{H.~Abdi}, \bibinfo{title}{{The Bonferonni and
  {\v{S}}id{\'{a}}k Corrections for Multiple Comparisons}},
  \bibinfo{journal}{Encyclopedia of Measurement and Statistics}
  (\bibinfo{year}{2007}) \bibinfo{pages}{103--107}.

\bibitem[{Heusinger et~al.(2020)Heusinger, Raab, and Schleif}]{Heusinger}
\bibinfo{author}{M.~Heusinger}, \bibinfo{author}{C.~Raab},
  \bibinfo{author}{F.-M. Schleif}, \bibinfo{title}{{Passive Concept Drift
  Handling via Momentum Based Robust Soft Learning Vector Quantization}}, in:
  \bibinfo{editor}{A.~Vellido}, \bibinfo{editor}{K.~Gibert},
  \bibinfo{editor}{C.~Angulo}, \bibinfo{editor}{J.~D. {Mart{\'{i}}n Guerrero}}
  (Eds.), \bibinfo{booktitle}{Advances in Self-Organizing Maps, Learning Vector
  Quantization, Clustering and Data Visualization},
  \bibinfo{publisher}{Springer International Publishing},
  \bibinfo{address}{Cham}, \bibinfo{pages}{200--209}, \bibinfo{year}{2020}.

\bibitem[{Biehl et~al.(2019)Biehl, Abadi, G{\"{o}}pfert, and Hammer}]{Biehl.19}
\bibinfo{author}{M.~Biehl}, \bibinfo{author}{F.~Abadi},
  \bibinfo{author}{C.~G{\"{o}}pfert}, \bibinfo{author}{B.~Hammer},
  \bibinfo{title}{{Prototype-Based Classifiers in the Presence of Concept
  Drift: A Modelling Framework}}, in: \bibinfo{booktitle}{Advances in
  Self-Organizing Maps, Learning Vector Quantization, Clustering and Data
  Visualization - Proceedings of the 13th International Workshop, {WSOM+} 2019,
  Barcelona, Spain, June 26-28, 2019}, \bibinfo{pages}{210--221},
  \bibinfo{year}{2019}.

\bibitem[{Montiel et~al.(2018)Montiel, Read, Bifet, and
  Abdessalem}]{skmultiflow}
\bibinfo{author}{J.~Montiel}, \bibinfo{author}{J.~Read},
  \bibinfo{author}{A.~Bifet}, \bibinfo{author}{T.~Abdessalem},
  \bibinfo{title}{{Scikit-Multiflow: A Multi-output Streaming Framework}},
  \bibinfo{journal}{Journal of Machine Learning Research} \bibinfo{volume}{1}
  (\bibinfo{year}{2018}) \bibinfo{pages}{1--5}.

\bibitem[{Cornu{\'{e}}jols(2010)}]{Cornuejols2010}
\bibinfo{author}{A.~Cornu{\'{e}}jols}, \bibinfo{title}{{On-Line Learning: Where
  Are We So Far?}}, in: \bibinfo{booktitle}{Lecture Notes in Computer Science
  (including subseries Lecture Notes in Artificial Intelligence and Lecture
  Notes in Bioinformatics)}, vol. \bibinfo{volume}{6202 LNAI},
  \bibinfo{publisher}{Springer}, \bibinfo{pages}{129--147},
  \bibinfo{year}{2010}.

\bibitem[{Baena-Garc{\'{i}}a et~al.(2006)Baena-Garc{\'{i}}a, {Del
  Campo-{\'{A}}vila}, Fidalgo, Bifet, Gavald{\`{a}}, and
  Morales-Bueno}]{Baena006}
\bibinfo{author}{M.~Baena-Garc{\'{i}}a}, \bibinfo{author}{J.~{Del
  Campo-{\'{A}}vila}}, \bibinfo{author}{R.~Fidalgo},
  \bibinfo{author}{A.~Bifet}, \bibinfo{author}{R.~Gavald{\`{a}}},
  \bibinfo{author}{R.~Morales-Bueno}, \bibinfo{title}{{Early Drift Detection
  Method}}, \bibinfo{journal}{Fourth International Workshop on Knowledge
  Discovery from Data Streams} \bibinfo{volume}{6} (\bibinfo{year}{2006})
  \bibinfo{pages}{77--86}.

\bibitem[{Gama et~al.(2004{\natexlab{b}})Gama, Medas, Castillo, and
  Rodrigues}]{Gama2004}
\bibinfo{author}{J.~Gama}, \bibinfo{author}{P.~Medas},
  \bibinfo{author}{G.~Castillo}, \bibinfo{author}{P.~Rodrigues},
  \bibinfo{title}{{Learning with Drift Detection}}, in:
  \bibinfo{editor}{A.~L.~C. Bazzan}, \bibinfo{editor}{S.~Labidi} (Eds.),
  \bibinfo{booktitle}{Advances in Artificial Intelligence -- SBIA 2004},
  \bibinfo{publisher}{Springer Berlin Heidelberg}, \bibinfo{address}{Berlin,
  Heidelberg}, \bibinfo{pages}{286--295}, \bibinfo{year}{2004}{\natexlab{b}}.

\bibitem[{Bifet and Gavald{\`{a}}(2009)}]{Bifet2009}
\bibinfo{author}{A.~Bifet}, \bibinfo{author}{R.~Gavald{\`{a}}},
  \bibinfo{title}{{Adaptive Learning from Evolving Data Streams}}, in:
  \bibinfo{booktitle}{Advances in Intelligent Data Analysis VIII},
  \bibinfo{publisher}{Springer Berlin Heidelberg}, \bibinfo{address}{Berlin,
  Heidelberg}, \bibinfo{pages}{249--260}, \bibinfo{year}{2009}.

\bibitem[{{Oza, Nikunj, Russell} and Oza(2005)}]{DBLP:conf/smc/Oza05}
\bibinfo{author}{S.~{Oza, Nikunj, Russell}}, \bibinfo{author}{N.~Oza},
  \bibinfo{title}{{Online bagging and boosting}}, in:
  \bibinfo{booktitle}{Proceedings of the IEEE International Conference on
  Systems, Man and Cybernetics, Waikoloa, Hawaii, USA, October 10-12, 2005},
  vol.~\bibinfo{volume}{3}, \bibinfo{publisher}{IEEE},
  \bibinfo{pages}{2340--2345}, \bibinfo{year}{2005}.

\bibitem[{Losing et~al.(2017{\natexlab{b}})Losing, Hammer, and
  Wersing}]{Losing2017}
\bibinfo{author}{V.~Losing}, \bibinfo{author}{B.~Hammer},
  \bibinfo{author}{H.~Wersing}, \bibinfo{title}{{Self-adjusting memory: How to
  deal with diverse drift types}}, \bibinfo{journal}{IJCAI International Joint
  Conference on Artificial Intelligence}  (\bibinfo{year}{2017}{\natexlab{b}})
  \bibinfo{pages}{4899--4903}.

\bibitem[{Street and Kim(2001)}]{Street2001SEAGenerator}
\bibinfo{author}{W.~N. Street}, \bibinfo{author}{Y.~Kim}, \bibinfo{title}{A
  Streaming Ensemble Algorithm (SEA) for Large-scale Classification}, in:
  \bibinfo{booktitle}{Proceedings of the Seventh ACM SIGKDD International
  Conference on Knowledge Discovery and Data Mining}, KDD '01,
  \bibinfo{publisher}{ACM}, \bibinfo{address}{New York, NY, USA},
  \bibinfo{pages}{377--382}, \bibinfo{year}{2001}.

\bibitem[{Domingos and Hulten(2000)}]{Domingos.2000}
\bibinfo{author}{P.~M. Domingos}, \bibinfo{author}{G.~Hulten},
  \bibinfo{title}{{Mining high-speed data streams}}, in:
  \bibinfo{booktitle}{Proceedings of the sixth ACM SIGKDD international
  conference on Knowledge discovery and data mining, Boston, MA, USA, August
  20-23}, \bibinfo{pages}{71--80}, \bibinfo{year}{2000}.

\bibitem[{Hulten et~al.(2001)Hulten, Spencer, and
  Domingos}]{Hulten2001Hyperplane}
\bibinfo{author}{G.~Hulten}, \bibinfo{author}{L.~Spencer},
  \bibinfo{author}{P.~Domingos}, \bibinfo{title}{Mining Time-changing Data
  Streams}, in: \bibinfo{booktitle}{Proceedings of the Seventh ACM SIGKDD
  International Conference on Knowledge Discovery and Data Mining}, KDD '01,
  \bibinfo{publisher}{ACM}, \bibinfo{address}{New York, NY, USA},
  \bibinfo{pages}{97--106}, \bibinfo{year}{2001}.

\bibitem[{Bifet et~al.(2018)Bifet, Gavald\`a, Holmes, and
  Pfahringer}]{Bifet2018MOA}
\bibinfo{author}{A.~Bifet}, \bibinfo{author}{R.~Gavald\`a},
  \bibinfo{author}{G.~Holmes}, \bibinfo{author}{B.~Pfahringer},
  \bibinfo{title}{Machine Learning for Data Streams with Practical Examples in
  MOA}, \bibinfo{publisher}{MIT Press}, \bibinfo{year}{2018}.

\bibitem[{Bifet et~al.(2013)Bifet, Pfahringer, Read, and Holmes}]{Bifet.2013}
\bibinfo{author}{A.~Bifet}, \bibinfo{author}{B.~Pfahringer},
  \bibinfo{author}{J.~Read}, \bibinfo{author}{G.~Holmes},
  \bibinfo{title}{{Efficient Data Stream Classification via Probabilistic
  Adaptive Windows}}, in: \bibinfo{booktitle}{Proceedings of the 28th Annual
  ACM Symposium on Applied Computing}, SAC '13, \bibinfo{publisher}{ACM},
  \bibinfo{address}{New York, NY, USA}, \bibinfo{pages}{801--806},
  \bibinfo{year}{2013}.

\bibitem[{Oza and Russell(2001)}]{Oza2001}
\bibinfo{author}{N.~C. Oza}, \bibinfo{author}{S.~Russell},
  \bibinfo{title}{{Experimental comparisons of online and batch versions of
  bagging and boosting}}, \bibinfo{journal}{Proceedings of the seventh ACM
  SIGKDD international conference on Knowledge discovery and data mining - KDD
  '01}  (\bibinfo{year}{2001}) \bibinfo{pages}{359--364}.

\end{thebibliography}

\end{document}